\documentclass[sigconf]{acmart}

\newtheorem{definition}{Definition}
\newtheorem{exmp}{Example}[section]
\usepackage{algorithm}
\usepackage{algorithmicx}
\usepackage{algpseudocode}
\usepackage{caption}
\newtheorem{theorem}{Theorem}[section]
\usepackage{subcaption}
\usepackage{comment}
\usepackage{adjustbox}

\usepackage{array}
\newcolumntype{P}[1]{>{\centering\arraybackslash}p{#1}}
\newcolumntype{M}[1]{>{\centering\arraybackslash}m{#1}}
\usepackage{mathrsfs}
\usepackage{tabularx}
\usepackage{booktabs}
\usepackage{multirow}
\usepackage{rotating}
\usepackage{xcolor}
\usepackage{amsmath}

\graphicspath{ {./images/} }

\AtBeginDocument{%
  \providecommand\BibTeX{{%
    \normalfont B\kern-0.5em{\scshape i\kern-0.25em b}\kern-0.8em\TeX}}}

\setcopyright{acmcopyright}
\copyrightyear{2024} 
\acmYear{2024} 
\setcopyright{acmlicensed}\acmConference[SIGMOD '24]{Proceedings of the 2024 International Conference on Management of Data}{June 06, 2024}{Santiago, Chile}
\acmBooktitle{Proceedings of the 2024 International Conference on Management of Data (SIGMOD '24), June 06, 2024, Santiago, Chile}
\acmPrice{15.00}

\usepackage{xcolor}
\definecolor{highlight}{RGB}{255,0,0}

\begin{document}

\title{Certain and Approximately Certain Models for Statistical Learning}
\author{Cheng Zhen}
\affiliation{%
 \institution{Oregon State University}
  \streetaddress{}
  \city{Corvallis}
  \state{Oregon}
 \country{}
  \postcode{97330}
}
\email{zhenc@oregonstate.edu}

\author{Nischal Aryal}
\affiliation{%
 \institution{Oregon State University}
  \streetaddress{}
  \city{Corvallis}
  \state{Oregon}
 \country{}
  \postcode{97330}
}
\email{aryaln@oregonstate.edu}

\author{Arash Termehchy}
\affiliation{%
 \institution{Oregon State University}
  \streetaddress{}
  \city{Corvallis}
  \state{Oregon}
 \country{}
  \postcode{97330}
}
\email{termehca@oregonstate.edu}

\author{Alireza Aghasi}
\affiliation{%
 \institution{Oregon State University}
  \streetaddress{}
  \city{Corvallis}
  \state{Oregon}
 \country{}
  \postcode{97330}
}
\email{alireza.aghasi@oregonstate.edu}

\author{Amandeep Singh Chabada}
\affiliation{%
 \institution{Oregon State University}
  \streetaddress{}
  \city{Corvallis}
  \state{Oregon}
 \country{}
  \postcode{97330}
}
\email{chabadaa@oregonstate.edu}
%
\renewcommand{\shortauthors}{Zhen and Aryal, et al.}

%
\begin{abstract}
Real-world data is often incomplete and contains missing values. To train accurate models over real-world datasets, users need to spend a substantial amount of time and resources imputing and finding proper values for missing data items. In this paper, we demonstrate that it is possible to learn accurate models directly from data with missing values for certain training data and target models. We propose a unified approach for checking the necessity of data imputation to learn accurate models across various widely-used machine learning paradigms. We build efficient algorithms with theoretical guarantees to check this necessity and return accurate models in cases where imputation is unnecessary. Our extensive experiments indicate that our proposed algorithms significantly reduce the amount of time and effort needed for data imputation without imposing considerable computational overhead.
\end{abstract}





\maketitle
\section{Introduction}\label{sec1}
The performance of a machine learning (ML) model relies substantially on the quality of its training data. 
Real-world training data often contain a considerable number of examples with missing values, i.e., {\it incomplete data}.
One may train an ML model by ignoring the training examples with missing values.
This approach, however, may significantly reduce the accuracy of the resulting model as it may lose out on some useful examples \cite{van2018flexible}.

To address the problem of training over incomplete data, users usually replace each missing data item with a value, i.e., data imputation, and train their models over the resulting {\it repaired data}.
To repair incomplete data, users must figure out the mechanisms and causes of data missingness, e.g., completely at random or based on observed values of some features \cite{10.1093/biomet/63.3.581}.
Based on this mechanism, they build a (statistical) model for missing data and replace the missing values with some measurements defined over this model, e.g., mean. 
Users may also leverage a variety of ML models to repair missing values, e.g., tree-based or linear regression \cite{little2002statistical}.
Researchers have shown that the desired imputation method may vary depending on the downstream ML task \cite{NEURIPS2021_5fe8fdc7}.
Hence, it is often challenging to find a model of data missingness that results in an accurate ML model for a downstream task \cite{NEURIPS2021_5fe8fdc7}.
The aforementioned steps of finding a missingness mechanism, constructing an accurate missingness model, and finding the right statistical measurement(s) for imputation usually require a significant amount of time and manual effort.
Surveys indicate that most users spend about 80\% of their time on data preparation and repair \cite{krishnan2016activeclean,neutatz2021cleaning}. 

{Researchers have recently shown that one may learn accurate Datalog rules \cite{picado2020learning} and K-nearest neighbor classifier \cite{karlavs2020nearest,DBLP:conf/icdt/FanK22} over a training dataset without cleaning and repairing it.} 
Generally speaking, these methods check whether incomplete or inconsistent examples influence the target model.
If this is not the case, they return the model learned over the original training data.    
This approach may save significant time and effort spent repairing data.

However, it is not clear whether the methods above can be used to check the necessity of data repair for other ML models.
As opposed to learning Datalog rules or K-nearest neighbors, training popular ML models usually requires optimizing a continuous loss function.
Moreover, these methods detect the necessity of data repair only for classification problems and do not handle learning over missing data for regression models.
Also, each of these methods handles a single ML model.
Due to the relatively large number and variety of ML models, one would ideally like to have a single approach to the problem of learning over data with missing values for multiple types of ML models.


In this paper, we aim to develop a general approach for learning accurate ML models over training data with missing values without any data repair.
We focus on ML models that optimize loss functions over continuous spaces, which arguably contain the most popular ML models.
We formally define the necessity of data repair for learning accurate models over training data with missing values.
Our methods efficiently detect whether data repair is needed to learn accurate models.
If data repair is not necessary, they learn effective models over the original training data. 
Particularly, we make the following contributions: 
\begin{itemize}
  \item We formally define the conditions where data repair is not needed for training optimal models over incomplete data for a large group of ML models (Section \ref{sec:certain-models}).
    \item We prove necessary and sufficient conditions for learning an optimal model without repairing incomplete data for {\it linear regression}. Based on these conditions, we design an efficient algorithm for 1) checking the existence of the optimal model, and 2) learning the optimal model if it exists (Section \ref{sec:linear-regression}).
    
    \item We prove necessary and sufficient conditions for learning an optimal model without repairing incomplete data for {\it linear Support Vector Machine (SVM)}, a popular {\it classification} ML model. We present an efficient algorithm for checking and then learning the optimal model if it exists (Section \ref{sec:linear-svm}). 
    
    \item Linear SVM models only learn linear classifiers, limiting their representation power in nonlinear spaces. We prove necessary and sufficient conditions for learning an optimal model without repairing incomplete data for two popular {\it kernel SVMs}. Then we provide algorithms to check and then learn the optimal models for {\it each} kernel SVM (Section \ref{sec:kernel-SVM})
    
    \item We formalize the notion of certain models for {\it Deep Neural Networks} (DNNs). Due to the non-convexity of the loss functions in DNNs, it is challenging to design an algorithm that efficiently finds the optimal model for them.
    We prove the necessary conditions for having certain models for DNNs in some special cases (Section \ref{sec:DNN}). 

    \item {It might not be possible to learn an optimal model over incomplete data without any data repair. Hence, we introduce and formally define the condition under which it is possible to learn a model that is sufficiently close to an optimal model over incomplete data without any repair. We propose novel and efficient algorithms to check for the existence of these models over linear regression and SVM (Section \ref{sec:Approximate})}. 
     
    \item  We conduct experiments to show cost savings in data cleaning and program execution time compared to mean imputation, a deep learning-based imputation algorithm, and a benchmark framework across real-world datasets with random corruption. We also extend the comparison to diverse real-world datasets with inherent missing values, yielding results consistent with randomly corrupted datasets. Our studies show that our algorithms significantly reduce data repair costs when optimal or approximately optimal models can be learned over incomplete data and introduce minimal computational overhead in other cases (Section \ref{sec:experimental-evaluation}).
\end{itemize}

\section{Background}\label{sec:Background}

\subsection{Supervised Learning}
\label{sec:supervised-learning}
In this section, we review ML terminology and notations.

\begin{table}[h]
  \caption{A training dataset for rain prediction}

  \label{tab:missing_value}
  \begin{tabular}{cccl}
    \toprule
    & Temperature(F)  & Humidity($\%$) & Rainfall\\
    \midrule
    Seattle & 65 & 80 & 1\\
    New York  & 50 & $null$ & -1\\
    \bottomrule
  \end{tabular}
\end{table}
\vspace{-3mm}
\paragraph{{\bf Dataset}} In ML, we work with a relation consisting of a finite number of attributes and tuples. {\it For instance, the relation shown in \autoref{tab:missing_value} has two tuples and three attributes.} For an ML problem, a relation with tuples and attributes is generally referred to as a {\bf dataset} with {\bf rows} and {\bf columns}. In supervised learning, an ML model takes certain columns from a dataset as input and makes predictions for a designated output column.

\paragraph{{\bf Features}} The columns of the dataset we provide as input to an ML model are called features. {\it In \autoref{tab:missing_value}, Temperature and Humidity are the two features that provide information on atmospheric conditions}. We denote a single feature as \textbf{z} and $d$ features as $[\textbf{z}_{1}, ..., \textbf{z}_{d}]$. The {\it domain} of feature $\mathbf{z_i}$ is the set of values that appear in feature $\mathbf{z_i}$. To simplify our exposition, we assume that the domain of all values in a feature is the set of real numbers $\mathbb{R}$.

\paragraph{{\bf Label}} The column of the dataset we want an ML model to make predictions on is called a label. {\it In our example, given current atmospheric conditions we want to predict chances of Rainfall. Therefore Rainfall is the label column, and it takes on two possible values: -1 to denote No Rain and 1 for Rain}. We represent a single label as $y$ and the entire label column, consisting of $n$ labels, as a vector $\textbf{y} = [y_{1}, ..., y_{n}]$.



\paragraph{{\bf Training Example}} We refer to a row in the dataset as a training example. {\it In \autoref{tab:missing_value}, we observe two examples, Seattle and New York}. We denote a single training example as \textbf{x}. For $n$ training examples, a {\bf training set} is a collection of an input matrix $\textbf{X} =[\textbf{x}_{1}, ..., \textbf{x}_{n}]^{T}$ and a corresponding label vector $\textbf{y} = [y_{1}, ..., y_{n}]^{T}$. Each training example with $d$ features in $\textbf{X}$ can be expressed as a vector $\textbf{x}_i = [x_{i1}, ..., x_{id}]$, where $x_{ij}$ represents the $j^{th}$ feature in the $i^{th}$ example.

\paragraph{{\bf Target Function}}
We define the domain of examples as $\mathcal{X}$ and the domain of labels as $\mathcal{Y}$. For $n$ examples and $d$ features, we assume, $\mathcal{X}$ and $\mathcal{Y}$ are $\mathbb{R}^{n\times d}$, and $\mathbb{R}^{n}$, respectively. A target function $f(\textbf{X}, \textbf{w})$ transforms feature inputs into label outputs based on model $\textbf{w}$, represented as $f(\textbf{X}, \textbf{w}): \mathcal{X} \rightarrow \mathcal{Y}$. Here, model $\textbf{w}$ is a real-valued vector parameterizing the space of target functions that map from $\mathcal{X}$ to $\mathcal{Y}$. For instance, consider a single training example $\mathbf{x_i}$ consisting of $d$ features.  Given a vector of real numbers {\bf w}, the target function may be a {\it linear transformation} of the example $\mathbf{x_i}$, i.e $f(\mathbf{x_i},\mathbf{w})= w_1\cdot x_{i1} + ... + w_d\cdot x_{id}$.
\vspace{-3mm}
\begin{figure}[H]
  \begin{subfigure}{0.45\linewidth}
  \includegraphics[width = \linewidth]{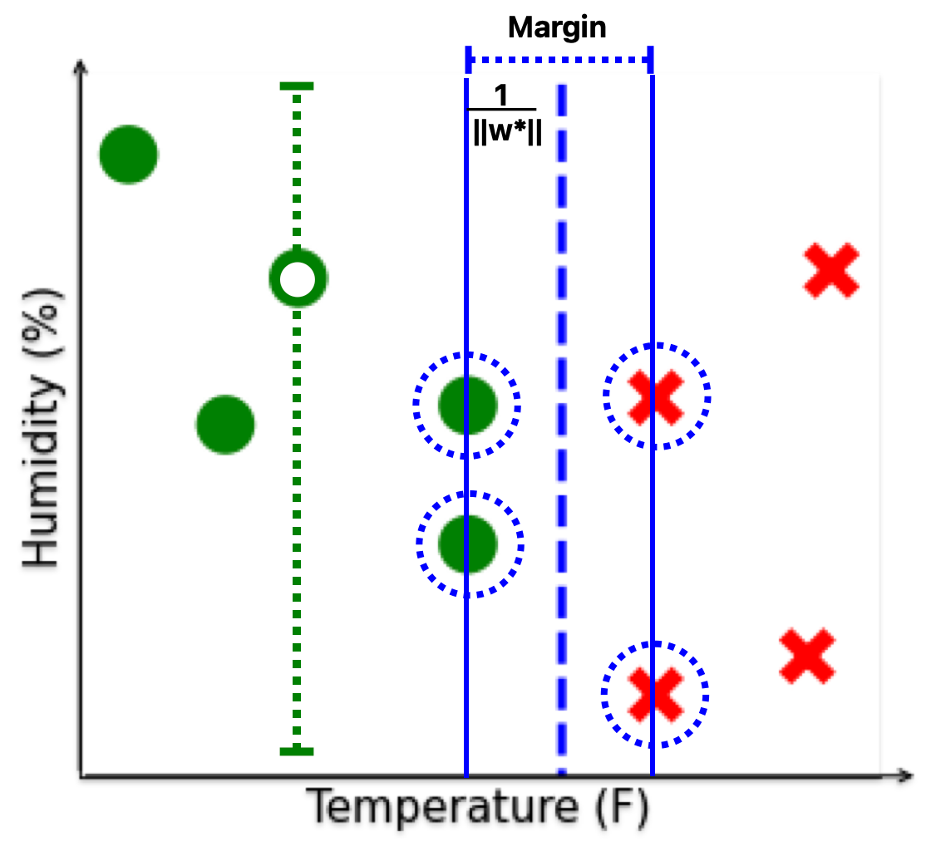}
  \caption{Data cleaning is not needed}
  \label{fig:1_a}
  \end{subfigure}
\begin{subfigure}{0.44\linewidth}
  \includegraphics[width = \linewidth]{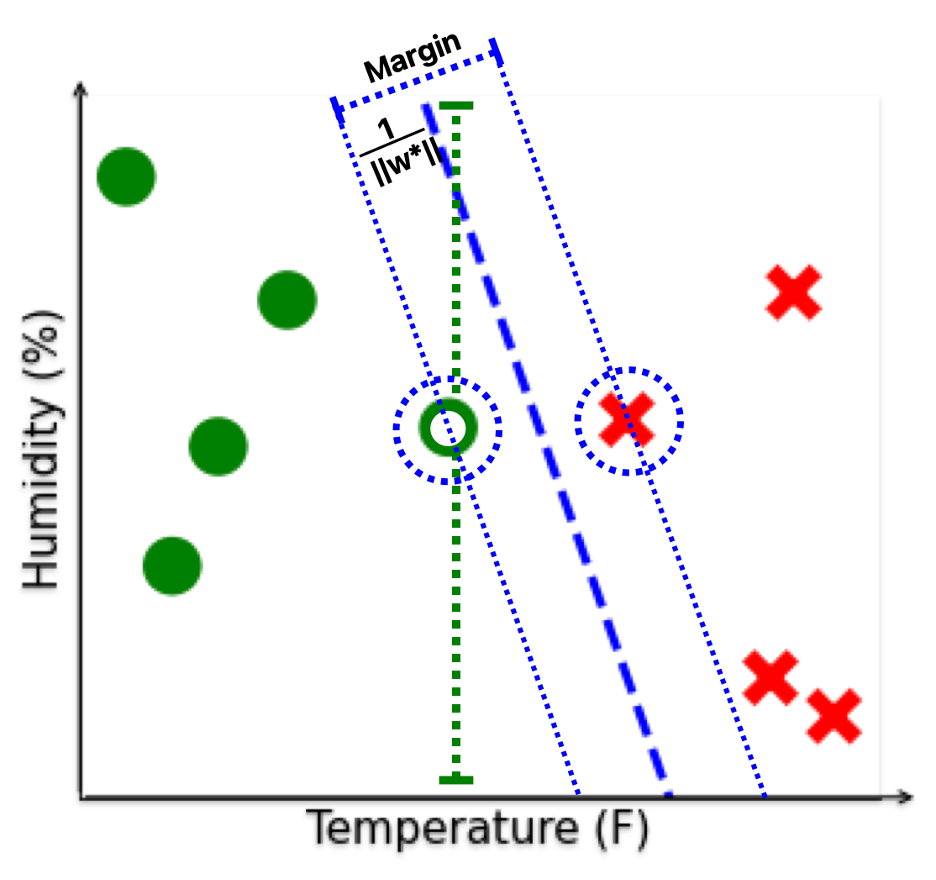}
  \caption{Data cleaning is needed}
  \label{fig:1_b}
  \end{subfigure}
  \vspace{-3mm}
   \caption{Data cleaning may not always be necessary}
    \label{fig:1}
\end{figure}
\vspace{-5mm}

\begin{exmp}
Consider \autoref{fig:1_a}, which uses a popular ML algorithm called Support Vector Machine (SVM). The goal is to learn a linear boundary (blue rectangle) for rain prediction using temperature and humidity features from different cities (examples $\mathbf{X}$). The boundary (margin) separates the examples based on their Rain outcomes. The target function transforms all examples to one of the two possible $\textbf{y}$ values [1, -1]. The approximation of the {\bf target function} is $f(\mathbf{X},\mathbf{w})$=$\mathbf{w}^T\mathbf{X}$. 
\end{exmp}

\paragraph{{\bf Loss function}} A loss function, $\mathcal{L}$, is defined as a mapping of prediction for an example $\mathbf{x_i}$, i.e., $f(\mathbf{x_i},\mathbf{w})$, with its corresponding label $y_i$ to a real number $l \in \mathbb{R}$. $l$ captures the similarity between $f(\mathbf{x_i},\mathbf{w})$ and $y_i$. The exact form of the loss function varies between ML problems. One reasonable measure to capture similarity is to get the difference between prediction $f(\mathbf{x_i},\mathbf{w})$ and actual label $y_i$. Aggregating over the $n$ examples in the input matrix ($\mathbf{X}$), we find the {\it overall loss function}, $L$: $L(f(\mathbf{X},\mathbf{w}), \mathbf{y}) = \frac{1}{n} \sum_{i=1}^{n} \mathcal{L}(f(\mathbf{x_i},\mathbf{w}), y_i)$ = $\frac{1}{n} \sum_{i=1}^{n}(f(\mathbf{x_i},\mathbf{w}) - y_i)^2$. For the rest of the paper, we will refer to the {\it `overall loss function'} as the {\bf loss function} since we will be working with a matrix of examples rather than individual examples.
\begin{exmp}
     For the SVM in \autoref{fig:1}, the \textbf{loss function}, L, is defined as $L(f(\mathbf{X},\mathbf{w}), \mathbf{y}) = \frac{1}{2}\|\mathbf{w}\|^2_2 + C\sum_{i=1}^{n}\max\{0, 1-y_i\mathbf{w}^T\mathbf{x}_i\}$. 

     \noindent
     Here, $\mathbf{w}^T\mathbf{x}_i$ comes from the \textbf{target function} and represents the model's prediction for an example $\mathbf{x}_i$. The actual label is denoted as $y_i$. When the prediction and the label have the same sign, they are similar, therefore the loss is lower. The notation $|| \cdot ||^2_2$ represents the squared Euclidean norm, and $C \in (0, +\infty)$ is a tunable parameter.
\end{exmp}

\paragraph{{\bf Classification and Regression}} Supervised learning is divided into two types of ML problems. In a classification problem, the label domain $\mathcal{Y}$ consists of discrete values {\it (such as Rain(1) or No Rain(-1))}. Whereas if the label domain consists of continuous values {\it (e.g. inches of Rainfall)}, then it is a regression problem.

\paragraph{{\bf Model Training}} Taking an input matrix $\mathbf{X}$, a label vector $\mathbf{y}$, a targte function $f$, and a loss function $L$, the goal of training is to find an optimal model $\mathbf{w}^*$ that minimizes the {\bf training loss}, i.e., $\mathbf{w}^* = \arg \min\limits_{\mathbf{w}\in \mathcal{W}} L(f(\mathbf{X},\mathbf{w}),\mathbf{y})$. 

\begin{exmp}
     For the SVM in \autoref{fig:1}, \textbf{optimal model} $\textbf{w}^*$ is the model that creates the widest margin between the example with different labels (red and green) while ensuring accurate predictions, to minimize training loss.
\end{exmp}

\subsection{Missing Values and Repairs}\label{sec:missing-value-repairs}
In this section, we formally define concepts for missing value repair.

\paragraph{{\bf Missing values}} Any $x_{ij}$ is a missing value (MV) if it is unknown (marked by \textit{null}). We use the term {\bf incomplete example} to refer to an example with missing values, and {\bf incomplete feature} for a feature that contains missing values. Conversely, we use the terms {\bf complete feature} and {\bf complete example} to describe features and examples that are free of missing values. We further denote the set of incomplete examples as $MV(\mathbf{x}) = \{\mathbf{x}_i | \exists x_{ij},  x_{ij} = \textit{null} \}$, and the set of incomplete features as $MV(\mathbf{z}) = \{\mathbf{z}_j | \exists x_{ij},  x_{ij} = \textit{null} \}$.

\begin{exmp}
In \autoref{tab:missing_value}, the Humidity feature is an \textbf{incomplete feature} while the Temperature feature is a \textbf{complete feature}. Similarly, the New York example is an \textbf{incomplete example}, and the Seattle example is a \textbf{complete example}
\end{exmp}

\paragraph{{\bf Repair}}A repair is a complete version of the raw data where all missing values (MV) are imputed i.e. replaced with values from the domain of features or examples (Subsection \ref{sec:supervised-learning}). More formally: 
\begin{definition}(Repair)
\label{definition_repair}
  For an input matrix $\textbf{X}$ having missing values (MV), $\textbf{X}^{r}$ is a repair to $\textbf{X}$ if 1) $dimension(\textbf{X}^{r}) = dimension(\textbf{X})$, 2) $\forall x^r_{ij} \in \textbf{X}^r, x^r_{ij} \neq \textit{null}$, and 3) $\forall x_{ij}\neq\textit{null}$,  $x^{r}_{ij} = x_{ij}$.
\end{definition}

\begin{exmp}
In Table \ref{tab:missing_value}, the Humidity feature for the New York \textbf{example} has a missing value. From Definition \autoref{definition_repair}, replacing the missing data with a value (e.g. $90$) yields a \textbf{repair} (\textbf{X}$^r$). However, deleting the humidity feature, which eliminates the missing value, is not a repair since it changes dimension(X).

\end{exmp}

\paragraph{{\bf Set of possible repairs}} The range of values that can be used to replace missing values is large. Consequently, a large number of repairs may exist. We denote this set of all possible repairs as  $\textbf{X}^{R}$. For brevity, we refer to {\it`a value replacing the missing value'} as a {\bf repairing value}.

\section{Certain Models}\label{sec:certain-models}
In this section, we formally define certain models that minimize training loss irrespective of how missing data is repaired.

\begin{definition}\label{definition:certain-model}(Certain Model) A model $\textbf{w}^*$ is a certain model if:

\begin{equation}\label{eq:1}
    \forall \textbf{X}^{r} \in \textbf{X}^{R},\textbf{w}^* = \mathop{\arg \min}\limits_{\textbf{w}\in \mathcal{W}} L(f(\textbf{X}^{r}, \textbf{w}), \textbf{y})
\end{equation}
\end{definition}
\noindent
{\it Where $\textbf{X}^{r}$ is one possible repair, $\textbf{X}^{R}$ is the set of all possible repairs and $L(f(\textbf{X}^{r}, \textbf{w}), \textbf{y})$ is the loss function}
\vspace{3mm}

\noindent
\textbf{Definition's intuition}: Intuitively, Definition \ref{definition:certain-model} says that if a model is optimal (minimizes the training loss) for all possible repairs, this model is a certain model. 

\begin{exmp}
  Consider the ML problem in Figure \ref{fig:1}. Figures \ref{fig:1_a} and \ref{fig:1_b} display two sets of training examples with a missing humidity value, possibly due to a malfunctioning sensor. The green dashed line represents the range of possible values for the incomplete feature (empty circle). In Figure \ref{fig:1_a}, the incomplete example does not touch the blue rectangle in any possible repair ($\textbf{X}^{r} \in \textbf{X}^{R}$). Therefore, the model (decision boundary: blue dashed line) is optimal for all repairs. Hence, a certain model ($\textbf{w}^*$) exists. But, in Figure \ref{fig:1_b}, since the example may touch the blue rectangle in many repairs, the optimal model changes from one repair to another and certain models do not exist.
 \end{exmp}
 \vspace{-2mm}
 \noindent
\textbf{Advantages of finding certain models}: To repair incomplete data users may resort to methods such as deleting data (e.g., entire examples or features), potentially leading to information loss. Another option is data imputation, which requires additional effort and domain expertise ~\cite{van2018flexible}. Regardless of how well these data repair techniques are constructed, they may produce suboptimal results, i.e., the repaired data is not the ground truth~\cite{liu2021Adaptive}. However when a certain model exists, {\it imputing missing data is unnecessary} since this model is optimal for all possible repairs. Therefore users may save a significant amount of time and effort by finding certain models. 
{Users may ignore missing values 
in practice to investigate the properties of the trained model. Nonetheless, there is no guarantee that their trained model is accurate. The concept of certain models ensures cases for which this approach leads to accurate models.}




\noindent
{\textbf{Prevalence of certain models:}} Certain models may not often exist from the restrictive definition (a model is optimal for {\it all} repairs). However, when they exist, we save a significant amount of time and resources. Furthermore, these savings significantly grow as the number of datasets increases alongside the rapid expansion of the ML community utilizing these datasets for model training.

\noindent
\textbf{Problems:} We aim to solve our problem of finding certain models by solving the following sub-problems.

\begin{enumerate}
    \item \textbf{Certain Model Checking}: Given {\bf inputs} (1) a training set consisting of a feature matrix $\textbf{X}$ potentially with missing values and a label vector $\textbf{y}$ (2) a target function $f(\textbf{X}, \textbf{w})$ and (3) a loss function $L$. The first problem is to determine \textit{whether a certain model $\textbf{w}^*$ exists} that minimizes the training loss $L(f(\textbf{X}, \textbf{w}), \textbf{y})$ for all repairs ($\forall \textbf{X}^{r} \in \textbf{X}^{R}$) to the incomplete dataset. If a certain model ($\textbf{w}^*$) exists, it implies that data imputation is unnecessary.
    \item \textbf{Certain Model Learning}: If a certain model exists, then the second problem is \textit{learning a certain model} ($\textbf{w}^*$), given a training set, loss function, and a target function as {\bf inputs}. This certain model {\bf output} can be used for downstream tasks.
\end{enumerate}

\noindent
{\textbf{Minimal overhead of not finding certain models:} When certain models exist users {\it do not have to spend any effort} in repairing missing data. When certain models do not exist, the effort to check for them may appear wasteful. Therefore, an ideal solution would require minimal time to check for certain models even when they do not exist. Consequently, the {\it overhead of checking certain models} is negligible compared to the significant time and resources users may save by {\it finding certain models}.

\noindent
\textbf{Baseline Algorithm:} Given Equation \ref{eq:1},  a \textit{baseline algorithm} for checking and learning a certain model is: (1) learning \textit{possible models} from all possible repairs one by one, and (2) a certain model exists if all repairs share at least one mutual optimal model. Here, the set of possible repairs is often large (Subsection \ref{sec:missing-value-repairs}). Therefore, {\it learning models from all repairs may be incredibly slow}. More precisely, if we denote the training time for learning one model as $\mathcal{O}(T_{train})$, the baseline algorithm's complexity grows in proportion to the size of all possible repairs ($\textbf{X}^R$). This results in a complexity of $\mathcal{O}(|\textbf{X}^R| * T_{train})$, where $|\textbf{X}^R|$ represents the total number of possible repairs. Therefore, we aim to find {\it efficient algorithms to check for certain models in multiple ML problems}.

\section{Certain Models for Linear Regression}\label{sec:linear-regression}
{\bf Linear regression} is a popular and classic ML model. 
It assumes a linear relationship between feature input (\textbf{X}) and label output (\textbf{y}). The difference between the model's prediction and actual label, $\textbf{X}\textbf{w} - \textbf{y}$, is the residue $\textbf{e} = [e_1, ..., e_n]$. 

The loss function (Section \ref{sec:Background}) for linear regression is $L(f(\textbf{X}, \textbf{w}), \textbf{y})$ $ = ||\textbf{X}\textbf{w}-\textbf{y}||^2_2$. Here, $\| \cdot \|^2_2$ represents the squared Euclidean norm.

\subsection{Conditions For Having Certain Models}
\label{subsec:problem-formulation-linear}
Based on the definition of certain models (Definition \autoref{definition:certain-model}), the \textbf{certain model $\textbf{w}^*$ for linear regression} is defined as: 

\begin{equation}
\label{equation:linear_problem_formulation}
\forall \textbf{X}^{r} \in \textbf{X}^{R}, \textbf{w}^* = \mathop{\arg \min}\limits_{\textbf{w}\in \mathcal{W}} ||\textbf{X}^{r}\textbf{w}-\textbf{y}||^2_2
\end{equation}
\noindent
{\it Where $\textbf{X}^{r}$ is one possible repair to the input matrix $\textbf{X}$, $\textbf{X}^{R}$ is set of all possible repairs and $||\textbf{X}^{r}\textbf{w}-\textbf{y}||^2_2$ is the loss function}
\vspace{2mm}

Linear regression finds a model $\mathbf{w}^* \in \mathbb{R}^d$ such that the linear combination of all feature vectors, $w_1^*\textbf{z}_1 + ... + w_d^*\textbf{z}_d$,  has the shortest Euclidean distance to the label vector $\textbf{y}$, i.e., the minimum training loss. Intuitively, a certain model exists when this Euclidean distance is independent of any incomplete features $\textbf{z}_j, j \in MV(\textbf{z})$. 

To formalize this intuition and avoid checking for all possible repairs, we introduce Theorem~\ref{theorem;linearregression}. Given an input matrix with $n$ examples and $d$ features, $\textbf{X} \in \mathbb{R}^{n \times d}$, we denote a missing-value-free ({\bf c}omplete) matrix $\textbf{X}_c \in \mathbb{R}^{n \times m}$ as a submatrix $(m < d)$ of the input matrix. $\textbf{X}_c$ only consists of the $m$ complete features $\textbf{z}_j$ from $\textbf{X}$, $ \textbf{z}_j \notin MV(\textbf{z})$. Performing linear regression with $\textbf{X}_c$ and the labels $\textbf{y}$, we get the model $\textbf{w}^*_{c} \in \mathbb{R}^m$. To facilitate subsequent analysis, we introduce another model $\textbf{w}^{\bullet}$ by expanding $\textbf{w}^*_c $ from $\mathbb{R}^m$ to $ \mathbb{R}^d$ by appending $(d - m)$ zero coefficients corresponding to incomplete features. For example, if the columns 2 and 4 in $\textbf{X} \in \mathbb{R}^{4}$ contain missing values, and $\textbf{w}_c^* = [1, 1]^T$, we create $\textbf{w}^{\bullet}$ by expanding $\textbf{w}_c^*$ to $\mathbb{R}^{4}$ and inserting zeros in the second and fourth entries. This process results in an {\it expanded model}, $\textbf{w}^{\bullet} = [1,0,1,0]^T$. This step aligns the linear coefficients between $\textbf{X}_c$ and $\textbf{X}^r$, simplifying the following theorems and proof. 

\begin{lemma}
\label{lemma;linearregression}
If a certain model $\textbf{w}^*$ exists, $\forall \textbf{z}_j \in MV(\textbf{z})$, the corresponding coefficient $w_j^* = 0$. In other words, if a certain model exists, $\textbf{w}^{\bullet}$ is a certain model. 
\end{lemma}

\begin{proof}

The gradient of linear regression model's loss function is $\nabla L(\textbf{w}) = \frac{2}{n}\sum_{i=1}^{n}(\textbf{w}^T\textbf{x}_i - y_i)\textbf{x}_i$. Since linear regression has a {\it convex} loss function with respect to model $\textbf{w}$, a certain model $\textbf{w}^*$ exists if and only if $\nabla L(\textbf{w}^*) = \textbf{0}$ for all repairs, i.e., 
\begin{equation}\label{eq:2}
     \forall \textbf{X}^r \in \textbf{X}^R, \forall j \in \{1, ..., d\},\frac{2}{n}\sum_{i=1}^{n}(\textbf{w}^{*T}\textbf{x}^r_i - y_i)x^r_{ij} = 0
\end{equation}
\noindent
    For $\textbf{z}_j \in MV(\textbf{z})$, we can split the gradient in Equation \ref{eq:2}: $\forall \textbf{X}^r \in \textbf{X}^R$, 
\begin{equation}\label{eq:4}
     \frac{2}{n}[\sum_{x_{ij} \neq \textit{null}}(\textbf{w}^{*T}\textbf{x}^r_i - y_i)x_{ij} + \sum_{x_{ij} = \textit{null}}(\textbf{w}^{*T}\textbf{x}^r_i - y_i)x_{ij}^r]= 0
\end{equation}
To guarantee that the second summations term in Equation \ref{eq:4} is equal to $0$ regardless of the value of the repair $x_{ij}^r$, $\textbf{w}^{*T}\textbf{x}^r_i - y_i = w^*_1x^r_{i1}+...+w^*_jx^r_{ij}+...+w^*_dx^r_{id} - y_i$ must equal $0$ for all repairs. This leads to a condition of $w^*_j = 0$. Given this, $\textbf{w}^{\bullet}$ is a certain model because it is trained with all complete features. 
\end{proof}

Based on this Lemma~\ref{lemma;linearregression}, we have the following result.

\begin{theorem}
\label{theorem;linearregression}
A certain model exists if and only if $;\forall \textbf{z}_j \in MV(\textbf{z})$, these conditions are met:
1) $\forall x_{ij} = \textit{null}$, $e_i = 0$; 2) 
$\sum_{x_{ij} \neq \textit{null}} x_{ij} \cdot e_i = 0$.
\end{theorem}

\begin{proof}

To reduce notations, we reformulate the gradient in Equation \ref{eq:2} to the form of the inner product: 
\begin{equation}\label{eq:3}
    \forall \textbf{X}^r \in \textbf{X}^R, \forall j \in \{1, ..., d\}, \langle\textbf{X}^r\textbf{w}^* - \textbf{y}, \textbf{z}^r_j\rangle = 0
\end{equation}

{\it First, we prove the necessity}. For all complete features $\textbf{z}_j \notin MV(\textbf{z})$, $\langle\textbf{X}_c\textbf{w}_c^* - \textbf{y}, \textbf{z}_j\rangle = 0$ holds from the definition of $\textbf{w}_c^*$. Since residue $\textbf{e}$ = $\textbf{X}^r\textbf{w}^{\bullet} - \textbf{y} = \textbf{X}_c\textbf{w}_c^* - \textbf{y}$, it is trivial that $\forall \textbf{z}_j \notin MV(\textbf{z}), \forall \textbf{X}^r \in \textbf{X}^R, \langle\textbf{X}^r\textbf{w}^{\bullet} - \textbf{y}, \textbf{z}^r_j\rangle = \langle\textbf{r}, \textbf{z}^r_j\rangle = 0$ because $\textbf{e}$ is orthogonal to all complete features that are used to train $\textbf{w}_c^*$. From the two conditions in Theorem~\ref{theorem;linearregression}, we find: $\forall \textbf{z}_j\in MV(\textbf{z}), \forall \textbf{X}^r \in \textbf{X}^R, \langle\textbf{X}^r\textbf{w}^{\bullet} - \textbf{y}, \textbf{z}^r_j\rangle = 0$. Hence, Equation \ref{eq:3} holds, justifying the necessity.

\textit{Then, we prove sufficiency by contradiction}. Assume that a certain model exists, but at least one condition in Theorem~\ref{theorem;linearregression} is not satisfied. Based on Lemma \ref{lemma;linearregression}, $\textbf{w}^{\bullet}$ is the certain model where $\forall \textbf{z}_j \in MV(\textbf{z}), w_j^{\bullet} = 0$. However, breaking either condition in Theorem 5.1 leads to the existence of a repair to an incomplete feature $\textbf{z}^{r1}_j, \textbf{z}_j\in MV(\textbf{z})$ such that $\langle \textbf{r}, \textbf{z}^{r1}_j\rangle = \langle\textbf{X}^r\textbf{w}^{\bullet} - \textbf{y}, \textbf{z}^r_j\rangle \neq 0$. Because the feature repair $\textbf{z}^{r1}_j$ is not orthogonal to the fitting residue $\textbf{e}$, we may find a model other than $\textbf{w}^{\bullet}$ whose training loss is smaller than $||\textbf{r}||_2$. In other words, $\textbf{w}^{\bullet}$ is not an optimal model with respect to the repair $\textbf{X}^{r1}$. This violates the definition that a certain model is optimal for all repairs. As a result, the sufficiency of Theorem~\ref{theorem;linearregression} holds through the contradiction. 
\end{proof}

\subsection{Checking and Learning Certain Models}
Theorem~\ref{theorem;linearregression} says that a certain model exists for linear regression if and only if the residue vector $\textbf{e}$ is orthogonal to incomplete features. If a certain model exists, the incomplete features may be safely ignored without compromising the minimization of training loss since they do not contribute to a smaller training loss than $\textbf{e}$. 


Based on Theorem~\ref{theorem;linearregression}, we present Algorithm \autoref{alg:linear_regression}. Our algorithm has two major steps: 1) computing the residue vector $\textbf{e}$ along with expanded model $\textbf{w}^{\bullet}$ based on complete features, and 2) checking the orthogonality between $\textbf{e}$ and all incomplete features. Finally, we obtain a certain model when it exists by getting $\textbf{w}^{\bullet}$, in which the incomplete features are ignored by the zero linear coefficients.

The algorithm's time complexity is $\mathcal{O}(T_{train})$, which is significantly faster than the baseline we discuss in Section \ref{sec:certain-models}. The efficiency of our algorithm stems from its ability to check for certain models without traversing over all possible repairs.

\begin{algorithm}
\caption{Checking and learning certain model for Linear Regression}\label{alg:linear_regression}

\begin{algorithmic}
\State $MV(\textbf{z}) \gets \text{features with missing values (incomplete features)}$
\State $ \textbf{w}^{\bullet} \gets \text{expanded model trained with complete features}$
\State $ \textbf{e} \gets \text{fitting residue with complete features}$
\State $ n \gets \text{the number of training examples}$
\For {$\textbf{z}_j \in MV(\textbf{z})$}
\State $innerProduct \gets 0$
\For{$i=1,2,\ldots,n$}
	\If{$x_{ij} = \textit{null}$ AND $e_i \neq 0$}
    \State \Return \text{"Certain model does not exist"}
    \ElsIf{$x_{ij} \neq \textit{null}$}
    \State $innerProduct \gets innerProduct + x_{ij}*e_i$
    \EndIf
\EndFor
\If{$innerProduct \neq 0$}
    \State \Return \text{"Certain model does not exist"}
    \EndIf
\EndFor
\State \Return \text{"A certain model $\textbf{w}^{\bullet}$ exists"}
\end{algorithmic}
\end{algorithm}

\section{Certain Models for SVM}\label{sec:linear-svm}

Another widely used ML model is SVM. In this section, we are specifically interested in {\it linear} SVM, which aims to learn a linear decision boundary to classify examples. This decision boundary is of the form $\textbf{w}^T\textbf{x} = 0$.

A typical soft-margin SVM's loss function comprises of a loss term and a regularizer, $L(f(\textbf{X}, \textbf{w}), \textbf{y}) = \frac{1}{2}||\textbf{w}||^2_2 + C\sum_{i=1}^{n}max\{0, 1-y_i\textbf{w}^T\textbf{x}_i\}$. Here, the first term is the regularization, the second term is the {\it hinge loss}~\cite{gentile1998hingeloss}, and $C \in (0, +\infty)$ is a tunable parameter. {\bf Support vectors} are the closest training examples that decide a decision boundary, i.e. $(\textbf{x}_i, y_i)$ is a support vector if $y_i\textbf{w}^T\textbf{x}_i \leq 1$.

\subsection{Conditions For Having Certain Models}
Similar to the definition in Subsection \ref{subsec:problem-formulation-linear},\textbf{certain model, $\textbf{w}^*$, for SVM} is defined as: 
\begin{equation}
\label{equation:svm_problem_formulation}
    \forall \textbf{X}^{r} \in \textbf{X}^{R}, \textbf{w}^* =  \mathop{\arg \min}\limits_{\textbf{w}\in \mathcal{W}} [\frac{1}{2}||\textbf{w}||^2_2 + C\sum_{i=1}^{n}max\{0, 1-y_i\textbf{w}^T\textbf{x}_i\}]
\end{equation}
\noindent
\textit{Where $\mathbf{X}^r$ denotes one possible repair, and $\mathbf{X}^R$ is the set of all possible repairs. $\mathbf{x}_i$ is an input example with $d$ features, and $y_i$ is its corresponding label.  $\textbf{w}^T\textbf{x}_i$ comes from the target function and measures the proximity between the example $\textbf{x}_i$ and the decision boundary}
\vspace{2mm}

An SVM leverages support vectors to construct a decision boundary for classifying examples. Therefore, the existence of a certain model for an SVM implies that incomplete examples are not support vectors in any repairs (Lemma \ref{lemma:linearsvm1}). 

To formalize this intuition, we present Theorem \ref{theorem:linearsvm} to check for certain models. Similar to the notations used in Subsection \ref{subsec:problem-formulation-linear}, we denote a {\bf c}omplete matrix $\textbf{X}_c$ as a submatrix of $\textbf{X}$ that consists of all the {\bf c}omplete examples $\textbf{x}_i, \textbf{x}_i \notin MV(\mathbf{x})$. Similarly, we define a subvector $\textbf{y}_c$ to include all labels corresponding to these complete training examples. We denote the SVM model trained with these complete examples and labels as $\textbf{w}^{\diamond} = [w^{\diamond}_1, ..., w^{\diamond}_d]^T$.

\begin{lemma}\label{lemma:linearsvm1}
If a certain model $\textbf{w}^*$ exists, there are only two possible cases and they do not have any overlap. Case 1: none of the incomplete examples is a support vector with respect to $\textbf{w}^*$ in any repair, i.e., $\forall \textbf{X}^r \in \textbf{X}^R, \forall \textbf{x}_i \in MV(\mathbf{x})$, $y_i\textbf{w}^{*T}\textbf{x}_i^r > 1$. Case 2: $\exists \textbf{x}_i \in MV(\mathbf{x})$, $y_i\textbf{w}^{*T}\textbf{x}_i^r = 1$. Also, $\forall \textbf{z}_j \in MV(\mathbf{z}), w^*_j = 0$. And $\forall \textbf{x}_i \in MV(\mathbf{x})$, $y_i\textbf{w}^{*T}\textbf{x}_i^r \geq 1$.
\end{lemma}

\begin{proof}
We prove the lemma by contradiction. The \textit{sub-gradient of the loss function for SVM} is:
\[
    \frac{\partial L(f(\textbf{X}, \textbf{w}), \textbf{y})}{\partial \textbf{w}} =  \textbf{w} + C\cdot[\sum_{\{\textbf{x}_p|  y_p\textbf{w}^T\textbf{x}_p < 1\}}-y_p\textbf{x}_p + \sum_{\{\textbf{x}_p|  y_p\textbf{w}^T\textbf{x}_p = 1\}}- \alpha_p y_p\textbf{x}_p]
\]
where $0 \leq \alpha_p \leq 1$ are constrained parameters corresponding to the non-differentiable point in the hinge loss. SVM has a unique optimal model $\textbf{w}^*$. Therefore, a certain model exists if and only if the sub-gradient equals zero in all repairs: $ \exists 0 \leq \alpha_p \leq 1$,
\begin{equation}\label{eq:5}
  \forall \textbf{X}^r \in \textbf{X}^R, \textbf{w}^* + C\cdot[\sum_{\{\textbf{x}^r_p|  y_p\textbf{w}^{*T}\textbf{x}^r_p < 1\}}-y_p\textbf{x}^r_p + \sum_{\{\textbf{x}^r_p|  y_p\textbf{w}^{*T}\textbf{x}^r_p = 1\}}- \alpha_p y_p\textbf{x}^r_p] = \textbf{0} 
\end{equation}
Now, assume that we have neither Case 1 nor Case 2 described in Lemma \ref{lemma:linearsvm1}, but a certain model exists. To construct this main assumption, we need two separate sets of sub-assumptions. Sub-assumption 1: a certain model $\textbf{w}^*$ exists, and there is a repair $\textbf{X}^{r1}$ such that $\exists \textbf{x}_i \in MV(\mathbf{x})$, $y_i\textbf{w}^{*T}\textbf{x}_i^{r1} < 1$. Sub-assumption 2: a certain model $\textbf{w}^*$ exists where $\exists \textbf{z}_j \in MV(\mathbf{z}), w^*_j \neq 0$. $\forall \textbf{x}_i \in MV(\mathbf{x})$, $y_i\textbf{w}^{*T}\textbf{x}_i^r \geq 1$, and $\exists \textbf{x}_i \in MV(\mathbf{x})$, $y_i\textbf{w}^{*T}\textbf{x}_i^r = 1$. The goal is to show both sub-assumptions contradict and thus the main assumption contradicts.

Starting with sub-assumption 1, suppose ${j}^{th}$ feature value $x_{ij} = null$ in the incomplete training example $\textbf{x}_i$, one may easily find another repair $\textbf{X}^{r2}$ such that: 1)
$y_i\textbf{w}^{*T}\textbf{x}_i^{r2} < 1$, and 2) the only difference between $\textbf{X}^{r1}$ and $ \textbf{X}^{r2}$ is the repairing value to $x_{ij}$, i.e., $x_{ij}^{r1} \neq x_{ij}^{r2}$, and 3) $\nexists 0 \leq \alpha \leq 1$ that satisfy Equation \ref{eq:5} with respect to $ \textbf{X}^{r2}$. This is because the repairing value to $x_{ij}$ may be any value within $[-\infty, +\infty]$ while the last term in Equation \ref{eq:5}, $\sum_{\{\textbf{x}^r_p|  y_p\textbf{w}^{*T}\textbf{x}^r_p = 1\}}- \alpha_p y_p\textbf{x}^r_p$, is bounded by finite numbers since $0 \leq \alpha \leq 1$. Therefore, $\mathbf{w}^*$ is not an optimal model for $\textbf{X}^{r2}$, contradicting the sub-assumption 1.

Moving to sub-assumption 2, $\exists \textbf{z}_j \in MV(\mathbf{z}), w^*_j \neq 0$. Similar to the proof for sub-assumption 1, we can always find a repair such that no $0 \leq \alpha \leq 1$ satisfies Equation \ref{eq:5}. This is again because the term $\sum_{\{\textbf{x}_p|  y_p\textbf{w}^T\textbf{x}_p < 1\}}-y_p\textbf{x}_p$, is unbounded from the arbitrary repairing value, while the term $\sum_{\{\textbf{x}^r_p|  y_p\textbf{w}^{*T}\textbf{x}^r_p = 1\}}- \alpha_p y_p\textbf{x}^r_p$, is bounded by finite numbers since $0 \leq \alpha \leq 1$.

As a result, the proof is complete from the contradicting assumptions.


\end{proof}

\begin{lemma}\label{lemma:linearsvm2}
If a certain model exists by Case 1 in Lemma \ref{lemma:linearsvm1}, $\textbf{w}^{\diamond}$ is the certain model.
\end{lemma}

\begin{proof}
Based on Lemma \ref{lemma:linearsvm1}, if a certain model exists, none of the incomplete examples are support vectors with respect to $\textbf{w}^*$ in any repair. As a nice property of SVM, {\it removing examples that are not support vectors does not change the optimal SVM model}. Therefore, the model trained without incomplete examples is also the optimal model with respect to the full training set. As a result, when a certain model exists, $\textbf{w}^{\diamond}$ is the certain model. 
\end{proof}

\begin{lemma}\label{lemma:linearsvm3}
If a certain model exists by Case 2 in Lemma \ref{lemma:linearsvm1}, models trained with any repairs of $\mathbf{X}$ are certain models.
\end{lemma}

\begin{proof}
Based on Case 2 in Lemma \ref{lemma:linearsvm1}, when a certain model exists, every incomplete example $\mathbf{x}_i$ is either a support vector always standing exactly at the boundary of $y_i\textbf{w}^{*T}\textbf{x}_i = 1$, or never a support vector. In the first scenario, for any repair, we can always find a set of slackness variables $\alpha_i$ such that the optimality condition in Equation \ref{eq:5} holds for a single $\mathbf{w}^*$. Similarly, in the second scenario, if an incomplete example is never a support vector in any repair, training with any repairs leads to an identical model. Therefore, models trained with any repairs of $\mathbf{X}$ are certain models.
\end{proof}

\begin{theorem}
\label{theorem:linearsvm}
    A certain model exists if and only if one of the two sets of conditions below is met. Set 1: 1) $\forall \textbf{z}_j \in MV(\mathbf{z}), w^{\diamond}_j = 0$, 
    2) $\forall \textbf{x}_i \in MV(\mathbf{x})$, $y_i\sum_{x_{ij} \neq \textit{null}}$$ w^{\diamond}_j x_{ij} > 1$. Set 2: 1) training a model $\mathbf{w}'$ with a random repair $\mathbf{X}^{r'} \in \mathbf{X}^R$, $\forall \textbf{z}_j \in MV(\mathbf{z}), w'_j = 0$, 
    2) $\forall \textbf{x}_i \in MV(\mathbf{x})$, $y_i\sum_{x_{ij} \neq \textit{null}}$$ w'_j x_{ij} \geq 1$. 
\end{theorem}

\begin{proof}
We prove the necessity first. When conditions in Set 1 are satisfied, since $\textbf{w}^{\diamond}$ is the optimal model trained by removing incomplete examples, 
\begin{equation}\label{eq:8}
   \textbf{w}^{\diamond} + C[\sum_{\substack{\{\textbf{x}_i|  y_i\textbf{w}^{\diamond  T}\textbf{x}_i < 1, \\ \textbf{x}_i\notin MV(\mathbf{x})\}}}-y_i\textbf{x}_i + \sum_{\substack{\{\textbf{x}_i|  y_i\textbf{w}^{\diamond  T}\textbf{x}_i = 1, \\ \textbf{x}_i\notin MV(\mathbf{x})\}}}- \alpha_i y_i\textbf{x}_i]= \textbf{0}   
\end{equation}

Combining the two conditions in Set 1, we get  $\forall \textbf{x}_i\in MV(\mathbf{x}), \forall \textbf{X}^r \in \textbf{X}^R, y_i\textbf{w}^{\diamond}\textbf{x}^r_i > 1 $. Hence, the incomplete training examples are not support vectors with respect to $\textbf{w}^{\diamond}$ in any repair. Therefore, from Equation \ref{eq:8}, we have:
\[
  \forall \textbf{X}^r \in \textbf{X}^R,    \textbf{w}^{\diamond} + C[\sum_{\textbf{x}_i|  y_i\textbf{w}^{\diamond  T}\textbf{x}_i < 1}-y_i\textbf{x}_i + \sum_{\textbf{x}_i|  y_i\textbf{w}^{\diamond  T}\textbf{x}_i = 1}- \alpha_i y_i\textbf{x}_i] = \textbf{0}
\]
 Based on the previous discussion about Equation \ref{eq:5}, $\textbf{w}^{\diamond}$ is a certain model. When conditions in Set 2 hold, similarly, $\mathbf{w}'$ is optimal for all repairs based on Equation \ref{eq:5}. Hence, the necessity holds. 
 
We prove sufficiency by contradiction. For condition Set 1, assume that a certain model exists but at least one condition does not hold. Based on Lemma \ref{lemma:linearsvm2}, $\textbf{w}^{\diamond}$ is the certain model. Further, from Case 1 in Lemma \ref{lemma:linearsvm1}, $\forall \textbf{X}^r \in \textbf{X}^R, \forall \textbf{x}_i \in MV(\mathbf{x})$, $y_i\textbf{w}^{\diamond T}\textbf{x}_i^r > 1$. If condition 2) does not hold, i.e. $\exists \textbf{x}_p \in MV(\mathbf{x})$, $y_p\sum_{x_{pj} \neq \textit{null}}$$ w^{\diamond}_j x_{pj} \leq 1$. One may easily find a repair $\textbf{X}^{r1}$ such that $y_p\sum_{j = 1}^{d} w^{\diamond}_j x^{r1}_{pj} = y_p\textbf{w}^{\diamond T}\textbf{x}_p^{r1} \leq 1$. This contradicts Lemma \ref{lemma:linearsvm1}. Instead, if condition 1) does not hold, $\exists \textbf{z}_q \in MV(\mathbf{z}), w^{\diamond}_q \neq 0$. Suppose $x_{sq} = \textit{null}$ is the missing value in feature $\textbf{z}_q$, one may always find another repair $\textbf{X}^{r2}$ such that $y_s\textbf{w}^{\diamond T}\textbf{x}_s^{r2} = $ $y_sw^{\diamond}_1x_{s1}^{r2} + ... + y_sw^{\diamond}_qx_{sq}^{r2} + ... +y_sw^{\diamond}_dx_{sd}^{r2}\leq 1$. This is because the term $y_sw^{\diamond}_qx_{sq}^{r2}$ may be any value within $[-\infty, +\infty]$: the repairing value to the missing $x_{sq}$ may be any value within $[-\infty, +\infty]$, $y_s = \pm 1$ and $w^{\diamond}_q \neq 0$. Therefore, original assumptions contradict. For condition Set 1, similarly, assumptions also contradict, proving the sufficiency.

\end{proof}

\vspace{-5mm}
\subsection{Checking and Learning Certain Models}
\label{sec:algorithm-linear-svm}
 Theorem \ref{theorem:linearsvm} says that a certain model for SVM exists if and only if none of the incomplete training examples are support vectors. Therefore, these incomplete examples are \textit{redundant} when it comes to learning the decision boundary given other complete examples. 

Using\ref{theorem:linearsvm}, we propose Algorithm \ref{alg:linear_svm} with two major steps: 1) learning $\textbf{w}^{\diamond}$ from complete training examples, and checking the conditions in Set 1 in \ref{theorem:linearsvm} against $\textbf{w}^{\diamond}$. If a certain model exists, $\textbf{w}^{\diamond}$ is the certain model. 2) If certain models are not found in step 1, learning $\textbf{w}'$ from an arbitrary repair, and checking the conditions in Set 2 against $\textbf{w}'$. If a certain model exists from this step, $\textbf{w}'$ is the certain model. The algorithm's time complexity is $\mathcal{O}(T_{train})$ as training models is the dominant part compared to condition checking.  

\begin{algorithm}
\caption{Checking and learning certain models for linear SVM}\label{alg:linear_svm}
\begin{algorithmic}
\State $MV(\mathbf{z}) \gets \text{incomplete features}$
\State $MV(\mathbf{x}) \gets \text{incomplete examples}$
\State $ \textbf{w}^{\diamond} \gets 
\text{the model trained with complete training examples}$

\For {$\textbf{z}_j \in MV(\mathbf{z})$}
\If{$w_{j}^{\diamond} \neq 0$}

    \State $\text{Case 1} \gets \text{False}$
    
    \EndIf
\EndFor
\If{$\text{Case 1} \neq \text{False}$}
\For {$\textbf{x}_i \in MV(\mathbf{x})$}
\If{$y_i\sum_{x_{ij} \neq \textit{null}}$$ w^{\diamond}_j x_{ij} \leq 1$}

    \State $\text{Case 1} \gets \text{False}$
    
    \EndIf
\EndFor
\EndIf
\If{$\text{Case 1} \neq \text{False}$}
    \State \Return \text{"A certain model $\textbf{w}^{\diamond}$ exists"}
\Else
\State $ \textbf{w}' \gets 
\text{the model trained with an arbitrary repair}$
\For {$\textbf{z}_j \in MV(\mathbf{z})$}
\If{$w_{j}' \neq 0$}

    \State \Return \text{" Certain models do not exist"}
    \EndIf
\EndFor
\For {$\textbf{x}_i \in MV(\mathbf{x})$}
\If{$y_i\sum_{x_{ij} \neq \textit{null}}$$ w'_j x_{ij} < 1$}

    \State \Return \text{" Certain models do not exist"}
    
    \EndIf
\EndFor
\State \Return \text{"A certain model $\textbf{w}'$ exists"}
\EndIf
\end{algorithmic}
\end{algorithm}

\vspace{-3mm}
\section{Certain Models for Kernel SVM}\label{sec:kernel-SVM}

SVM models in Section \ref{sec:linear-svm} can only separate classes linearly, limiting their representation power in the nonlinear space. A natural approach to overcome this limitation is to use {\it kernel SVM}. 

Training a nonlinear model while maintaining the properties of linear SVM, a {\it kernel SVM} first maps the input feature vectors, denoted as $\textbf{X}$, into a higher-dimensional space, often referred to as the {\it kernel space}, through a non-linear transformation $\Phi$. After this transformation, the kernel SVM seeks to learn a linear SVM model within the kernel space. Therefore, the resulting model is non-linear with respect to the original feature space, while remaining linear within the kernel space.

However, transforming all training examples into kernel space is computationally expensive. To avoid this cost, {\bf kernel function} $k(\textbf{x}_1, \textbf{x}_2) = <\Phi(\textbf{x}_1), \Phi(\textbf{x}_2)>: \mathcal{X} \times \mathcal{X} \rightarrow \mathbb{R}$ offers a shortcut for computing inner products between two vectors in the kernel space without explicit transformation.

We presented the {\it primal problem} to linear {\bf SVM's model training} in Section \ref{sec:linear-svm}. Here, {\it to make use of kernel functions}, we present SVM training in terms of inner products through its {\it dual problem}.
\begin{equation}
\label{eqn:dual-problem}
\mathop{\max}\limits_{\textbf{a}\in \mathbb{R}^n} \sum_{i=1}^{n} a_i - \frac{1}{2}\sum_{i=1}^{n}\sum_{j=1}^{n}a_i a_j y_i y_j k(\textbf{x}_i, \textbf{x}_j) 
\end{equation}
\[
\textrm{s.t.} \quad  C \geq a_i \geq 0, i = 1, ..., n
\]
\[
\sum_{i=1}^{n}a_iy_i = 0
\]
Based on this dual formulation, one can show that $\textbf{w}^* = \sum_{i=1}^{n} a_i^* y_i \phi(\textbf{x}_i)$ where $\textbf{a}^* = [a_1^*, ..., a_n^*]^T$ is the solution to the dual problem. In Section \ref{sec:linear-svm}, a training example $(\textbf{x}_i, y_i)$ is a support vector in {\bf linear space} if $y_i\textbf{w}^{*T}\textbf{x}_i \leq 1$. Representing $\textbf{w}^{*T}$ by its dual form, a training example $(\textbf{x}_i, y_i)$ is a support vector in {\bf kernel space} if $y_i\sum_{j=1}^{n} a_j^* y_j k(\textbf{x}_i, \textbf{x}_j) \leq 1$.

\subsection{Conditions For Having Certain Models}
The kernel function transforms input data to a higher dimension while the SVM model remains linear. The linear properties of the kernel SVM are preserved within the kernel space. Hence, certain model conditions in Section \ref{sec:linear-svm} still apply. \textit{A certain model exists if and only if none of the incomplete examples are support vectors for any repair \textit{in the kernel space}}. 

We now formally present these conditions for kernel SVM. Following the same notations used in Section \ref{sec:linear-svm}, we use $\textbf{w}^{\diamond}$ to denote the model learned from $\textbf{X}_c$, the subset of data that only containing complete training examples, and $\textbf{y}_c$,  the corresponding labels. As derived from \autoref{eqn:dual-problem}, $\textbf{w}^{\diamond} = \sum_{\textbf{x}_j \notin MV(\mathbf{x})} a_{j}^{\diamond} y_j \phi(\textbf{x}_j)$. Hence, $\textbf{x}_i^r$, a repair to an incomplete training example, is a support vector in kernel space if $y_i\sum_{\textbf{x}_j \notin MV(\mathbf{x})} a_{j}^{\diamond} y_j k(\textbf{x}_i^r,\textbf{x}_j) \leq 1$. Therefore, the {\bf certain model conditions for kernel SVM} are represented as:
\begin{equation}
\label{equation:genereal_svm_problem_formulation}
\forall \textbf{x}_i \in MV(\mathbf{x}), \forall \textbf{X}^r \in \textbf{X}^R, y_i\sum_{\textbf{x}_j \notin MV(\mathbf{x})} a_{j}^{\diamond} y_j k( \textbf{x}_i^r, \textbf{x}_j) \geq 1
\end{equation}
Further, seeking the opportunity to avoid materializing all possible repairs, we reformulate the above condition to an optimization problem over possible repairs:

\begin{equation}\label{eq:9}
\forall \textbf{x}_i \in MV(\mathbf{x}), \mathop{\min}\limits_{\textbf{X}^r\in \textbf{X}^R}  y_i\sum_{\textbf{x}_j \notin MV(\mathbf{x})} a_{j}^{\diamond} y_j k( \textbf{x}_i^r, \textbf{x}_j) \geq 1    
\end{equation}


From the dual problem, we note that a complete example, $\textbf{x}_j, \textbf{x}_j \notin MV(\mathbf{x}),$ is a support vector if and only if the corresponding solution $a^{\diamond}_{j} \neq 0$. Hence, only complete examples that are support vectors play a role in Inequality \ref{eq:9}

In the following sections, we apply these general conditions for certain model existence in kernel SVM to popular kernel functions.

\subsection{Polynomial kernel}
\label{sec:kpoly}
The kernel function for a polynomial kernel is $k_{POLY}(\textbf{x}_i, \textbf{x}_j) = (\textbf{x}_i^T\textbf{x}_j + c)^\lambda$, where $\lambda = 1, 2, 3, ...$ is the degree of the polynomial and $c \geq 0$ is a free parameter tuning the impact of higher-degree versus lower-degree terms. 

We first intuitively look at how $k_{POLY}(\textbf{x}_i, \textbf{x}_j)$ remains the same value for all repairs. For an incomplete training example $\textbf{x}_i$ and a complete example $\textbf{x}_j$, $\textbf{x}_i^T\textbf{x}_j$ can be expanded to $x_{i1}\cdot x_{j1} + ... + x_{id}\cdot x_{jd}$. Suppose the $m^{th}$ feature value $x_{im}$ is missing in $\textbf{x}_i$, the inner product $\textbf{x}_i^T\textbf{x}_j$ goes to infinity when $x_{im} = +\infty$ or $-\infty$, unless the corresponding element $x_{jm}$ equals $0$, which ensures $x_{jm} \cdot x_{im} = 0$. Hence, in order to satisfy Inequality \ref{eq:9}, the {\bf set of support vectors}, $SV$ , for $\textbf{w}^{\diamond}$ should have zero entries at all incomplete features $\textbf{z}_m$. This condition enforces that the value for $k_{POLY}(\textbf{x}_i^r, \textbf{x}_j)$ {\it is independent of the missing value repairs}. We formalize these conditions in the following theorem.

\begin{theorem}
\label{theorem:svmkPOLY}
A certain model exists if and only if the two conditions are met: 1) $\forall \textbf{x}_j \in SV$, $\forall \textbf{z}_m \in MV(\mathbf{z})$, $x_{jm} = 0$, and 2) $\forall \textbf{x}_i \in MV(\mathbf{x})$, $y_i\sum_{\textbf{x}_j \in SV} a_{j}^{\diamond} y_j (\sum_{x_{iq} \neq \textit{null}} x_{iq}\cdot x_{jq} + c)^\lambda > 1 $
\end{theorem}

\begin{proof}
Necessity is trivial. Plugging condition 1) into condition 2), we get $\forall \textbf{x}_i \in MV(\mathbf{x})$, $\forall \textbf{X}^r \in \textbf{X}^R$, $y_i\sum_{\textbf{x}_j \in SV} a_{j}^{\diamond} y_j (\textbf{x}_i^{rT} \textbf{x}_j + c)^\lambda > 1 $. Since $\forall \textbf{x}_j \notin SV, a_{j} = 0$, we further get $\forall \textbf{x}_i \in MV(\mathbf{x})$, $\forall \textbf{X}^r \in \textbf{X}^R$, $y_i\sum_{\textbf{x}_j \notin MV(\mathbf{x})} a_{j}^{\diamond} y_j k_{POLY}(\textbf{x}^r_i, \textbf{x}_j) > 1$. As a result, a certain model exists as Inequality \ref{eq:9} is met.\\

To prove sufficiency, we first assume a certain model exists when condition 1) is not met, i.e., $\exists \textbf{z}_m \in MV(\mathbf{z}), \exists \textbf{x}_j \in SV, x_{jm} \neq 0$. As we discuss above, the inner product $\textbf{x}^{rT}_i\textbf{x}_j$ goes to infinity in repairs whose $x_{im}^r = +\infty$ or $-\infty$, hence not satisfying Inequality \ref{eq:9}. Then if we assume condition 1) is met but not condition 2), it is trivial that Inequality \ref{eq:9} is not satisfied either.
\end{proof}

\noindent
\textbf{Checking and Learning Certain Models:} Informally \autoref{theorem:svmkPOLY} says that a certain model for a polynomial kernel SVM (p-SVM) exists if (1) all the examples that are support vectors have zero entries for corresponding incomplete features and (2) all incomplete examples are not support vectors. Based on this theorem, Algorithm \ref{alg:kPOLY} efficiently checks and learns certain models. Similar to the algorithm for linear SVM in Section \ref{sec:algorithm-linear-svm}, if a certain model is determined to exist, $\textbf{a}^{\diamond}$ is exactly the certain model based on Lemma \ref{lemma:linearsvm2}. This algorithm's time complexity is also $\mathcal{O}(T_{train})$.

\begin{algorithm}
\caption{Checking and learning certain models for p-SVM}\label{alg:kPOLY}
\begin{algorithmic}
\State $MV(\mathbf{z}) \gets \text{incomplete features}$
\State $MV(\mathbf{x}) \gets \text{incomplete examples}$
\State $SV \gets \text{set of support vectors}$
\State $ \textbf{a}^{\diamond} \gets 
\text{the model trained with complete training examples}$

\For {$\textbf{z}_m \in MV(\mathbf{z})$}
\For {$\textbf{x}_j \in SV$}
\If{$x_jm \neq 0$}

    \State \Return \text{"Certain model does not exist"}
    
    \EndIf
\EndFor
\EndFor
\For {$\textbf{x}_i \in MV(\mathbf{x})$}
\If{$y_i\sum_{\textbf{x}_j \in SV} a_{j}^{\diamond} y_j (\sum_{x_{iq} \neq \textit{null}} x_{iq}\cdot x_{jq} + c)^\lambda \leq 1 $}

    \State \Return \text{"Certain model does not exist"}
    
    \EndIf
\EndFor
    \State \Return \text{"A certain model $\textbf{a}^{\diamond}$ exists"}
\end{algorithmic}
\end{algorithm}

\subsection{RBF kernel}
\label{sec:Krbf}
The RBF kernel function is $k_{RBF}(\textbf{x}_i, \textbf{x}_j) = exp(-\gamma||\textbf{x}_i - \textbf{x}_j||^2)$. This kernel function's transformation depends on the squared Euclidean distance between the two vectors $\textbf{x}_i$ and $\textbf{x}_j$.

To check if a certain model exists for the polynomial kernel, we derived conditions for {\it $k_{POLY}(\textbf{x}_i^r, \textbf{x}_j)$} to remain the {\it same} for all repairs. In contrast, {\it $k_{RBF}(\textbf{x}_i^r, \textbf{x}_j)$ changes} among repairs as the Euclidean distance between two vectors changes. Therefore, to check if a certain model exists for SVM with RBF kernel (RBF-SVM), we need to directly solve the minimization problem in Inequality \ref{eq:9}. 

However, this optimization problem is {\it not convex}, which means it is hard to find a method for checking certain models with theoretical guarantees. Nonetheless, we can still discover the {\it lower bound} ($lwb_i$) of the following optimization target:
\begin{equation}
 \forall \textbf{X}^r\in \textbf{X}^R, lwb_i \leq y_i\sum_{\textbf{x}_j \notin MV(\mathbf{x})} a_{j}^{\diamond} y_j k_{RBF}(\textbf{x}_i^r, \textbf{x}_j)
\end{equation}
 
For each missing value $x_{im}$, we denote the possible range of missing value repairs such that $x_{m}^{min}\leq x^r_{im} \leq x_m^{max}, \forall \textbf{z}_m \in MV(\mathbf{z}), \forall \textbf{X}^r \in \textbf{X}^R$. This range may come from {\it integrity constraint} for features: any value in a feature $\textbf{z}_m$ is between its minimum $x_{m}^{min}$ and maximum $x_{m}^{max}$. Now, we apply this lower bound idea to reformulate the general certain model conditions for kernel SVM from Inequality \ref{eq:9}.

\begin{lemma}\label{lemma:kernelsvm}
 For any kernel SVM, a certain model exists if 
 \begin{equation}\label{eq:kernelSVM}
  \forall \textbf{x}_i \in MV(\mathbf{x}), lwb_i = \sum_{\textbf{x}_j \notin MV(\mathbf{x})} \mathop{\min}\limits_{\textbf{x}_i^r\in \textbf{x}_i^R}\beta_{ij}k(\textbf{x}_i^r, \textbf{x}_j) > 1
\end{equation}
 where $\beta_{ij} = y_i a_{j}^{\diamond} y_j$ and
 \[
\mathop{\min}\limits_{\textbf{x}_i^r\in \textbf{x}_i^R}\beta_{ij}k(\textbf{x}_i^r, \textbf{x}_j) = 
\begin{cases}
 \beta_{ij} \mathop{\min}\limits_{\textbf{x}_i^r\in \textbf{x}_i^R}k(\textbf{x}_i^r, \textbf{x}_j) & \text{if $\beta_{ij} > 0$} \\
  \beta_{ij} \mathop{\max}\limits_{\textbf{x}_i^r\in \textbf{x}_i^R}k(\textbf{x}_i^r, \textbf{x}_j) & \text{if $\beta_{ij} < 0$}\\
  0 & \text{if $\beta_{ij} = 0$}
\end{cases}
\]

\end{lemma}

\begin{proof}
    Reformulating Inequality \ref{eq:9} to \[
    \forall \textbf{x}_i \in MV(\mathbf{x}), \mathop{\min}\limits_{\textbf{x}_i^r\in \textbf{x}_i^R} \sum_{\textbf{x}_j \notin MV(\mathbf{x})}  \beta_{ij} k(\textbf{x}_i^r, \textbf{x}_j) \geq 1 
    \] 
    where $\beta_{ij} = y_i a_{j}^{\diamond} y_j$. Given that the summations of minimums are always smaller or equal to the minimum of summations, we find the lower bound for the above optimization target as follows: 
    \[
lwb_i =  \sum_{\textbf{x}_j \notin MV(\mathbf{x})}  \mathop{\min}\limits_{\textbf{x}_i^r\in \textbf{x}_i^R}\beta_{ij}k(\textbf{x}_i^r, \textbf{x}_j)\leq \mathop{\min}\limits_{\textbf{x}_i^r\in \textbf{x}_i^R}\sum_{\textbf{x}_j \notin MV(\mathbf{x})} \beta_{ij}  k(\textbf{x}_i^r, \textbf{x}_j)
\]
where
\[
\mathop{\min}\limits_{\textbf{x}_i^r\in \textbf{x}_i^R}\beta_{ij}k(\textbf{x}_i^r, \textbf{x}_j) = 
\begin{cases}
 \beta_{ij} \mathop{\min}\limits_{\textbf{x}_i^r\in \textbf{x}_i^R}k(\textbf{x}_i^r, \textbf{x}_j) & \text{if $\beta_{ij} > 0$} \\
  \beta_{ij} \mathop{\max}\limits_{\textbf{x}_i^r\in \textbf{x}_i^R}k(\textbf{x}_i^r, \textbf{x}_j) & \text{if $\beta_{ij} < 0$}\\
  0 & \text{if $\beta_{ij} = 0$}
\end{cases}
\]
When $lwb_i > 1, \forall \textbf{x}_i \in MV(\mathbf{x})$, Inequality \ref{eq:9} holds.
\end{proof}

From Lemma \ref{lemma:kernelsvm}, we see the {\bf key to an efficient implementation} is to compute $\mathop{\min}\limits_{\textbf{x}_i^r\in \textbf{x}_i^R}k(\textbf{x}_i^r, \textbf{x}_j)$ and $\mathop{\max}\limits_{\textbf{x}_i^r\in \textbf{x}_i^R}k(\textbf{x}_i^r, \textbf{x}_j)$ without materializing repairs. We formalize this idea in Theorem 7.3.

\begin{theorem}\label{theorem:kRBF}
For the RBF kernel, the minimum and maximum kernel function values between an incomplete example and a complete example are as follows:
\begin{align*}
 \mathop{\min}\limits_{\textbf{x}_i^r\in \textbf{x}_i^R}k_{rbf}(\textbf{x}_i^r, \textbf{x}_j) = exp\{-\gamma\{\sum_{x_{im} = \textit{null}} MAX[(x_{m}^{max} - x_{jm})^2,\\ (x_{m}^{min} - x_{jm})^2] \\
    + \sum_{x_{im} \neq \textit{null}} (x_{im} - x_{jm})^2\}\} 
\end{align*}
\begin{align*}
 \mathop{\max}\limits_{\textbf{x}_i^r\in \textbf{x}_i^R}k_{rbf}(\textbf{x}_i^r, \textbf{x}_j) = exp\{-\gamma[ \sum_{x_{im} \neq \textit{null}} (x_{im} - x_{jm})^2]\}
\end{align*}
\end{theorem}

\begin{proof}
Since $k(\textbf{x}_i^r, \textbf{x}_j) = exp[-\gamma[(x^r_{i1} - x_{j1})^2 + ... + (x^r_{id} - x_{jd})^2]]$ by definition, the maximum of the kernel function value comes from the minimal $(x^r_{i1} - x_{j1})^2 + ... + (x^r_{id} - x_{jd})^2$. Therefore, $\mathop{\max}\limits_{\textbf{x}_i^r\in \textbf{x}_i^R}k(\textbf{x}_i^r, \textbf{x}_j)$ is achieved in the repair $\textbf{X}^{r1}$ such that $\forall x_{im} = \textit{null}, \textbf{x}^{r1}_{im} = x_{jm}$. Similarly, the minimum of the kernel function value comes from the maximal $(x^r_{i1} - x_{j1})^2 + ... + (x^r_{id} - x_{jd})^2$, corresponding to another repair $\textbf{X}^{r2}$ such that:
\[
\forall x_{im} = \textit{null}, x^{r2}_{im} = 
\begin{cases}
 x_m^{min} & \text{if $(x_{m}^{max} - x_{jm})^2 < (x_{m}^{min} - x_{jm})^2$} \\
 x_m^{max} & \text{if $(x_{m}^{max} - x_{jm})^2 \geq (x_{m}^{min} - x_{jm})^2$}
\end{cases}
\]
\end{proof}

\noindent
\textbf{Checking and Learning Certain Models:} 
Similar to the algorithm for the polynomial kernel, we can use Theorem \ref{theorem:kRBF} to check and learn certain model in $\mathcal{O}(T_{train})$ time.

\section{Certain Models for DNN}\label{sec:DNN}

DNNs are popular ML models for a wide variety of tasks such as natural language processing and image classification~\cite{alzubaidi2021review}.
Training a DNN involves solving complex non-convex optimization problems, making the discovery of an optimal model a challenging task ~\cite{cho2011analysis}. Finding a certain model for DNN adds another layer of difficulty because the certain model needs to be optimal for all repairs within the non-convex loss landscape.



Fortunately, some well-studied kernel SVMs have been shown to approximate DNNs~\cite{cho2011analysis}. Therefore, our goal in this section is to build on the conditions we prove for kernel SVMs in Section ~\ref{sec:kernel-SVM} to prove the {\bf conditions for having certain models for DNN}.

More specifically, we employ the {\it arc-cosine} kernel, which is used in SVM to approximate DNN's computation~\cite{cho2009kernel}. The justification behind this approximation stems from the following property. Feeding two input vectors $\textbf{x}_i$ and $\textbf{x}_j$ individually into a single-layer network with polynomial activation functions, we obtain the corresponding output vectors $\textbf{y}_i$ and $\textbf{y}_j$. Under some assumptions, the inner product between these two output vectors can be represented by the arc-cosine kernel function, i.e.,  $k_{arccos}(\textbf{x}_i, \textbf{x}_j) = <\textbf{y}_i, \textbf{y}_j>$ \cite{cho2009kernel}. This implies that the arc-cosine kernel function mimics the computation in a single-layer network. Then, iteratively performing kernel transformation, i.e., $<\phi(\phi(...\phi(\textbf{x}_i))), \phi(\phi(...\phi(\textbf{x}_j)))> $, should mimic the computation in a multi-layer network. The most basic arc-cos kernel function is defined by the inverse cosine of the dot product between two vectors divided by the product of their Euclidean norms, i.e. $k_{\text{arccos}}(\mathbf{x}_i, \mathbf{x}_j) = \cos^{-1}\left(\frac{\mathbf{x}_i \cdot \mathbf{x}_j}{|\mathbf{x}_i| \cdot |\mathbf{x}_j|}\right)$. By discovering the certain model conditions for SVM with the arc-cosine kernel (arccos-SVM), we approximate the certain model conditions for DNN. 


To check the existence of certain models, we need to solve the minimization problem in Inequality \ref{eq:9}. As the problem is non-convex, we find a lower bound ($lwb_i$) for the necessary condition:
\[
 \forall \textbf{X}^r\in \textbf{X}^R, lwb_i \leq y_i\sum_{\textbf{x}_j \notin MV(\mathbf{x})} a_{j}^{\diamond} y_j k_{arccos}(\textbf{x}_i^r, \textbf{x}_j)
\]

Following a similar approach as we describe in Subsection \ref{sec:Krbf}, we look for the lower bound defined in Lemma \ref{lemma:kernelsvm}. The key of this process is to find the minimum and maximum values for $k_{arccos}(\textbf{x}_i^r, \textbf{x}_j)$ for a pair of incomplete example $\textbf{x}_i$ and complete example $\textbf{x}_j$.

However, finding the minimum and maximum values for the arc-cosine kernel function in the presence of missing data is also challenging due to the non-convex nature of the problem. Nonetheless, when each incomplete example $\textbf{x}_i$ has only one missing value $x_{iz}$, the problem significantly simplifies. In the following proof, we show that any stationary point is a global minimum under this assumption. Therefore, our analysis focuses on training sets with one missing value per example. The investigation of scenarios with multiple missing values per example is left for future work. 

Following Theorem~\ref{theorem:kRBF}, we formalize this idea in Theorem~\ref{theorem:DNN}.

\begin{theorem}\label{theorem:DNN}
For arc-cos kernel, the maximum and the minimum kernel function values between an incomplete example ($\textbf{x}_i^r$) and a complete example ($\textbf{x}_j$) are as follows:
\begin{align*}
 \mathop{\max}\limits_{\textbf{x}_i^r\in \textbf{x}_i^R}k_{arccos}(\textbf{x}_i^r, \textbf{x}_j) = \pi - MAX[cos^{-1}(\frac{a}{c}), cos^{-1}(-\frac{a}{c})]
\end{align*}
\vspace{-5mm}
\begin{align*}
 \mathop{\min}\limits_{\textbf{x}_i^r\in \textbf{x}_i^R}k_{arccos}(\textbf{x}_i^r, \textbf{x}_j) = \pi - cos^{-1}(\frac{a^2\cdot d + b^2}{c\cdot \sqrt{a^2\cdot d^2 + b^2 \cdot d}})
\end{align*}
Suppose $x_{iz} = \textit{null}$. To simplify notations, we define $a = x_{jz}$, $b = \sum_{x_{ip} \neq \textit{null}} x_{ip} \cdot x_{jp}$, $c = ||\textbf{x}_j||$, and $d = \sum_{x_{ip} \neq \textit{null}} x_{ip}^2$.
\end{theorem}

\begin{proof}
Consider the arc-cosine term of the kernel function, $cos^{-1}(\frac{\textbf{x}^r_i \cdot \textbf{x}_j}{||\textbf{x}^r_i||\cdot||\textbf{x}_j||})$. Based on the property of the arc-cos function: 
\[
arg\mathop{\max}\limits_{\textbf{x}_i^r\in \textbf{x}_i^R}cos^{-1}(\frac{\textbf{x}^r_i \cdot \textbf{x}_j}{||\textbf{x}^r_i||\cdot||\textbf{x}_j||}) =arg\mathop{\min}\limits_{\textbf{x}_i^r\in \textbf{x}_i^R}\frac{\textbf{x}^r_i \cdot \textbf{x}_j}{||\textbf{x}^r_i||\cdot||\textbf{x}_j||}
\]
And vice versa. So we define a function
\[
f(x^r_{iz}) =\frac{a\cdot x^r_{iz} + b}{c \cdot \sqrt{(x^r_{iz})^2 + d}}
\]
Taking the derivative of the $f(x^r_{iz})$ with respect to the repair $x^r_{iz}$:
\[
f(x^r_{iz})' = \frac{a}{c \cdot \sqrt{(x^r_{iz})^2 + d}} - \frac{a\cdot (x^r_{iz})^2 + b\cdot x^r_{iz}}{c \cdot\sqrt{[(x^r_{iz})^2 + d]^3}}
\]
Through some simple derivation, one can find $f'(x^r_{iz}) \geq 0$ when $x^r_{iz} \leq \frac{a\cdot d}{b}$, and $f'(x^r_{iz}) \leq 0$ when $x^r_{iz} \geq \frac{a\cdot d}{b}$. This means that $x^r_{iz} = \frac{a\cdot d}{b}$ is the only stationary point that gives us the global minimum of $f(x^r_{iz})$. And the maximum is reached at either $x^r_{iz} = -\infty$ or $x^r_{iz} = +\infty$. Finally, bringing these maximum and minimum for $f(x^r_{iz})$ to the kernel function, we find $\mathop{\min}\limits_{\textbf{x}_i^r\in \textbf{x}_i^R}k_{arccos}(\textbf{x}_i^r, \textbf{x}_j)$ and $\mathop{\max}\limits_{\textbf{x}_i^r\in \textbf{x}_i^R}k_{arccos}(\textbf{x}_i^r, \textbf{x}_j)$ as outlined in Theorem 8.1.
\end{proof}

Theorem \ref{theorem:DNN} shows that the maximum and minimum values for the arc-cos kernel can be efficiently computed without materializing repairs. Further, plugging these values in Lemma \ref{lemma:kernelsvm}, we approximate a certain model condition for DNN that says a certain model exists if Inequality \ref{eq:kernelSVM} holds. 


\section{Approximately Certain Models}\label{sec:Approximate} 
The conditions for having certain models might be too restrictive for many datasets as it requires a single model to be optimal for all repairs of a dataset.
In practice, however, users are usually satisfied with a model that is {\it sufficiently close} to the optimal one. 
In this section, we leverage this fact and propose the concept of {\it approximately certain model}, which relaxes the conditions on certain models.
An approximately certain model is within a given threshold from every optimal model for each repair of the input dataset.
If there is an approximately certain model for a training task, users can learn over incomplete data and skip data cleaning. 
We also propose novel and efficient algorithms for finding approximately certain models for linear regression and SVM.




\subsection{Formal Definition}



\begin{definition}\label{definition:approximate-certain-model}(Approximately Certain Model) 
Given a user-defined threshold $e \geq 0$, the model $\mathbf{w}^{\approx}$ is an approximately certain model (ACM) if the following condition holds:
\begin{equation}\label{equation:apprximateCertainModel}
\forall \mathbf{X}^r \in \mathbf{X}^R, L(f(\mathbf{X}^{r}, \mathbf{w}^{\approx}), \mathbf{y}) - \min_{\mathbf{w} \in \mathcal{W}} L(f(\mathbf{X}^{r}, \mathbf{w}), \mathbf{y}) \leq e
\end{equation}
where $\textbf{X}^{r}$ is a possible repair, $\textbf{X}^{R}$ is the set of all possible repairs and $L(f(\textbf{X}^{r}, \textbf{w}), \textbf{y})$ is the loss function.
\end{definition}

\noindent
Definition~\ref{definition:approximate-certain-model} ensures that the training losses of ACMs are close to the minimal training loss for all repairs. Therefore, when $e$ is sufficiently small, ACMs are accurate for all repairs. In this scenario, data imputation is unnecessary and users can proceed with an ACM without compromising the model's performance significantly. Certain models are special cases of ACMs by setting $e = 0$.

\subsection{Learning ACMs Efficiently}

The condition in Definition~\ref{definition:approximate-certain-model} is equivalent to $g(\mathbf{w}') \leq e$ where 
\begin{equation}\label{equation:ACMcheckGivenModel}
g(\mathbf{w}') = \sup_{\mathbf{X}^r \in \mathbf{X}^R} h(\mathbf{w}', \mathbf{X}^r)
\end{equation}
and
\[
h(\mathbf{w}', \mathbf{X}^r) = L(f(\mathbf{X}^{r}, \mathbf{w}'), \mathbf{y}) - \min_{\mathbf{w} \in \mathcal{W}} L(f(\mathbf{X}^{r}, \mathbf{w}), \mathbf{y})
\]

If there is a model $\mathbf{w}' \in \mathcal{W}$ that satisfies this condition, it is an ACM. 
Hence, to find an ACM, we can check every $\mathbf{w}' \in \mathcal{W}$ for the condition in \ref{equation:ACMcheckGivenModel}. This is equivalent to checking $\min_{\mathbf{w}' \in \mathcal{W}}  g(\mathbf{w}') \leq e$.


\begin{lemma}\label{lemma:ACMconvex}
The problem $\min_{\mathbf{w}' \in \mathcal{W}}  g(\mathbf{w}') \leq e$
is convex for every model whose loss function $L(f(\mathbf{X}, \mathbf{w}), \mathbf{y})$ is convex with respect to $\mathbf{w}$.
\end{lemma}

\begin{proof}
Denote
\begin{equation}\label{eq:ACMlemmaproof}
h(\mathbf{w}', \mathbf{X}^r) = L(f(\mathbf{X}^{r}, \mathbf{w}'), \mathbf{y}) - \min_{\mathbf{w} \in \mathcal{W}} L(f(\mathbf{X}^{r}, \mathbf{w}), \mathbf{y})    
\end{equation}

$h(\mathbf{w}', \mathbf{X}^r)$ is convex with respect to $\mathbf{w}'$ for every $\mathbf{X}^r \in \mathbf{X}^R$. Therefore, $g(\mathbf{w}') = \sup_{\mathbf{X}^r \in \mathbf{X}^R} h(\mathbf{w}', \mathbf{X}^r)$ is convex as point-wise supremum preserves convexity.
\end{proof}
The loss functions of many types of models including linear regression and SVM are convex with respect to $\mathbf{w}$.
Thus, Lemma~\ref{lemma:ACMconvex} reduces the problem of finding ACMs to a convex optimization problem for many types of models.
Nonetheless, this problem is still challenging to solve via common techniques, e.g., gradient descent, because computing $\nabla g(\mathbf{w})$ involves finding the supremum over a large set of possible repairs $\mathbf{X}^R$.
We can reduce the search for finding the supremum to a significantly smaller subset of repairs. 

\begin{definition}(Edge Repair)
  Assume each missing value $x_{ij}$ in $\mathbf{X}$ is bounded by an interval such that $x_{ij}^{min} \leq x_{ij} \leq x_{ij}^{max}$. A repair $\mathbf{X}^e$ is an edge repair if $x_{ij}^e = x_{ij}^{min}$ or $x_{ij}^{min}$ for all missing values $x_{ij}$. $\mathbf{X}^E$ denotes the set of all possible edge repairs $\mathbf{X}^e$.    
\end{definition}

\begin{theorem}\label{theorem:ACMedgeRepair}
For linear regression and SVM, we have

\[
g(\mathbf{w}') = \sup_{\mathbf{X}^e \in \mathbf{X}^E} h(\mathbf{w}', \mathbf{X}^e)
\]

\end{theorem}

\begin{proof}
To prove by contradiction, assume 1) $\mathbf{X}^r$ is not an edge repair because at least one element of $\mathbf{X}^r$ is not at the edge of the interval (i.e., $x_{ij}^{min} < x_{ij}^r < x_{ij}^{max}$), and 2) the repair $\mathbf{X}^r$ corresponds to the supremum in $g(\mathbf{w}')$:
\begin{equation}\label{equation:ACMlemmaLRproof}
 g(\mathbf{w}') = h(\mathbf{w}', \mathbf{X}^r)  
\end{equation}
Our goal is to show that there is an edge repair $\mathbf{X}^e$ such that
\[
h(\mathbf{w}', \mathbf{X}^e) > h(\mathbf{w}', \mathbf{X}^r) 
\]
contradicting to Equation \ref{equation:ACMlemmaLRproof}, the initial assumption.

To start, we recall two parts of $h(\mathbf{w}', \mathbf{X}^r)$ in Equation \ref{eq:ACMlemmaproof}. For Expanding the first part for linear regression:
\[
L(f(\mathbf{X}^{r}, \mathbf{w}'), \mathbf{y}) = \sum_{i = 1}^{n} (\mathbf{w}'^T\mathbf{x}_i^r - y_i)^2
\]
And for SVM:
\[
L(f(\mathbf{X}^{r}, \mathbf{w}'), \mathbf{y}) = \frac{1}{2}||\textbf{w}'||^2_2 + C\sum_{i=1}^{n}max\{0, 1-y_i\textbf{w}'^T\textbf{x}_i\}
\]
One can see that $L(f(\mathbf{X}^{r}, \mathbf{w}'), \mathbf{y})$ increases quadratically and becomes unbounded when $x_{ij}$ moves from $x_{ij}^r$ to an edge of the interval ($x_{ij}^{min}$ or $x_{ij}^{max}$), given that we have assumed intervals are sufficiently wide. 

On the contrary, the second part of $h(\mathbf{w}', \mathbf{X}^r)$ is bounded by a deterministic value for linear regression when $x_{ij}$ moves from $x_{ij}^r$ to an edge of the interval:
\[
\sup_{\mathbf{X}^r \in \mathbf{X}^R} [\min_{\mathbf{w} \in \mathcal{W}} L(f(\mathbf{X}^{r}, \mathbf{w}), \mathbf{y})] \leq \min_{\mathbf{w} \in \mathcal{W}} L(f(\mathbf{X}^{c}, \mathbf{w}), \mathbf{y})
\]

This upper bound is the regression residue with complete submatrix $\mathbf{X}^c$, defined in Section 4.1 of the paper. $\mathbf{X}^c$ is a submatrix of $\mathbf{X}$ created by removing all incomplete columns. The legitimacy of this upper bound is based on an interpretation of linear regression: projecting the label vector onto the column space of feature matrix. Adding column vectors to $\mathbf{X}^c$ never shrinks the column space, and thus never reduces the regression residue upon projection. 

Therefore, when $x_{ij}$ moves from $x_{ij}^r$ to $x_{ij}^e$, the first part becomes dominant and increases quadratically, leading to:
\[
h(\mathbf{w}', \mathbf{X}^e) > h(\mathbf{w}', \mathbf{X}^r) 
\]
The dominance of the first part also exists for SVM at some edge repairs. As a result, this contradicts to the original assumptions.

\end{proof}
Based on Theorem \ref{theorem:ACMedgeRepair}, we can compute $g(\mathbf{w}')$ by finding the supremum of $h(\mathbf{w}', \mathbf{X}^r)$ only from edge repairs. This approach is efficient for dataset with relatively small number of missing values.
However, it may take long for datasets with many missing values because the number of edge repairs is $2^{n(MV)}$ where $n(MV)$ is the number of missing values in $\mathbf{X}$. 

To accelerate finding ACMs for linear regression and SVM, Algorithm \ref{alg:ACM-sample} randomly samples edge repairs and estimates the supremum of $h(\mathbf{w}', \mathbf{X}^r)$. 
This estimation is reasonable when the number of samples $s$ is large. 
The algorithm's time complexity is $\mathcal{O}(k\cdot d \cdot n \cdot s)$, where $k$ stands for the number of iterations in gradient descent.

\begin{algorithm}
\caption{Learning ACM}\label{alg:ACM-sample}

\begin{algorithmic}
\State $s \gets \text{the number of edge repairs to sample}$
\State $e\gets \text{user-defined threshold for approximate optimality}$
\State $\mathbf{X}^E_{sample} \gets random.sample(\mathbf{X}^E, s)$\\
\Comment{randomly add $s$ edge repairs to the sample set}
\For{$\mathbf{X}^e \in \mathbf{X}^E_{sample}$}

\State $h(\mathbf{w}', \mathbf{X}^e) \gets L(f(\mathbf{X}^{e}, \mathbf{w}'), \mathbf{y}) - \min_{\mathbf{w} \in \mathcal{W}} L(f(\mathbf{X}^{e}, \mathbf{w}), \mathbf{y})$
\EndFor
\State $\mathbf{w}^\approx \gets arg\min_{\mathbf{w} \in \mathcal{W}} \sup_{\mathbf{X}^e \in \mathbf{X}^E_{sample}} h(\mathbf{w}', \mathbf{X}^e)$

\Comment{This optimization is solved by existing algorithms}

\If{$ g(\textbf{w}^\approx) \leq e$}
    \State \Return ${w}^\approx$ 
\Else
    \State \Return 
    \text{"Approximately certain models do not exist"}
\EndIf
\end{algorithmic}
\end{algorithm}

\begin{table*}[!t]
\centering
\caption{Details of Real World Dataset containing missing values}
\vspace{-3mm}
\label{tab:real-world-missing-charateristics}
\begin{tabular}{|c|c|c|c|c|c|}
\hline
Data Set & Task & Features & Training Examples & Missing Factor\\
\hline
{Breast Cancer} & Classification & 10 & 559 & 1.97\% \\
\hline
{Intel-Sensor} & Classification & 11 & 1850945 & 4.05\%  \\
\hline
{NFL} & Regression & 34 & 34302 & 9.04\%  \\
\hline
{Water-Potability} & Classification & 9 & 2620 & 39.00\%  \\
\hline
{Online Education} & Classification & 36 & 7026 & 35.48\%  \\
\hline
{COVID} & Regression & 188 & 60229 & 53.67\%  \\
\hline
{Air-Quality} & Regression & 12 & 7192 & 90.99\%  \\
\hline
{Communities} & Regression & 1954 & 1595 & 93.67\%  \\
\hline
\end{tabular}
\end{table*}

\subsection{ACMs for Regression With Guarantees}
For linear regression, if some conditions hold in the dataset, we can decompose the computation of the supremum for $h(\mathbf{w}', \mathbf{X}^e)$ in Theorem \ref{theorem:ACMedgeRepair} to each example and compute ACMs in linear time. 

\begin{theorem}\label{theorem:ACMLR}
For linear regression, if 
\begin{align*}
\forall i \in [1, ... , n], \forall \mathbf{x}_i^{e} \neq \mathbf{x}_i^{e*}, L(f(\mathbf{x}_i^{e*}, \mathbf{w}'), \mathbf{y_i}) & - L(f(\mathbf{x}_i^{e}, \mathbf{w}'), \mathbf{y_i}) \geq \\
&\min_{\mathbf{w} \in \mathcal{W}} L(f(\mathbf{x_i}_c, \mathbf{w}), \mathbf{y}) 
\end{align*}

where ${x_i}_c$ is created by ignoring features with missing values in $x_i$, then

\[
g(\mathbf{w}') = L(f(\mathbf{X}^{e*}, \mathbf{w}'), \mathbf{y}) - \min_{\mathbf{w} \in \mathcal{W}} L(f(\mathbf{X}^{e*}, \mathbf{w}), \mathbf{y})
\]

where 
\[
\forall i \in [1, ... , n], \mathbf{x}_i^{e*} = arg \max_{\mathbf{x}_i^e \in \mathbf{x}_i^E} (\mathbf{w}'^T\mathbf{x}_i^e - y_i)^2
\]
and $\mathbf{x}_i^E$ is the set of edge repairs for $x_i$.
\end{theorem}

\begin{proof}
Firstly, given a model $\mathbf{w}'$, the training loss with any repair $\mathbf{X}^r \in \mathbf{X}^R$  is bigger than $\min_{\mathbf{w} \in \mathcal{W}} L(f(\mathbf{X}_c, \mathbf{w}), \mathbf{y})$. This is because the training a linear regression model is projecting the label vector onto the column space, and adding any columns to $\mathbf{X}_c$ never increases the training loss compared to $\min_{\mathbf{w} \in \mathcal{W}} L(f(\mathbf{X}_c, \mathbf{w}), \mathbf{y})$. Further, each training example in the repair $\mathbf{X}^{e*}$ must maximize the training loss with respect to the corresponding because the total training loss is the summation of training losses from independent training examples. Therefore, $\mathbf{X}^{e*}$ is computed by finding $\mathbf{x}_i^{e*}$ separately.
\end{proof}

Because training examples are independent, $\mathbf{X}^{e*}$ maximizes the overall training loss if and only if each training example in $\mathbf{X}^{e*}$ maximizes the squared error for the example. Further, when the latter condition in the theorem holds, $h(\mathbf{w}', \mathbf{X}^{e*})$ is also the supremum of $h(\mathbf{w}', \mathbf{X}^e)$.
It is because this condition ensures that the training loss term is absolutely dominant in $h(\mathbf{w}', \mathbf{X}^e)$. 
This allows us to find the supremum edge repair for each training example individually. 

Algorithm \ref{alg:ACM-linear_regression} uses this result to efficiently compute ACMs for linear regression.
It uses the common gradient descent algorithm as $g(\mathbf{w})$ is convex. 
Its time complexity is $\mathcal{O}(k\cdot d \cdot n)$.
The latter condition in Theorem \ref{theorem:ACMLR} is checked in linear time.

\begin{algorithm}
\caption{Learning ACMs for Linear Regression}\label{alg:ACM-linear_regression}

\begin{algorithmic}
\State $\mathbf{w}^{(0)} \gets \mathbf{w}^{init}$
\State $t \gets 0$
\State $n \gets \text{the number of training examples}$
\State $e\gets \text{user-defined threshold for approximate optimality}$

\While{$||\nabla g(\textbf{w}^{(t)})|| > \epsilon$}
    \State $t \gets t + 1$
    \For {$i = 1, 2, ..., n$}
    \State $\mathbf{x}_i^{\text{e*}} \gets \operatorname*{argmax}_{\mathbf{x}_i^{\text{e}}} ||\mathbf{w}^{(t-1)T}\mathbf{x}_i^{\text{e}} -y_i||^2_2$


    \EndFor
	\State $\nabla g(\textbf{w}^{(t-1)}) \gets \nabla L(f(\mathbf{X}^{e*}, \mathbf{w}^{(t-1)})$
	\State $\mathbf{w}^{(t)} \gets \mathbf{w}^{(t-1)} - \eta \nabla g(\textbf{w}^{(t-1)})$
\EndWhile

\If{$ g(\textbf{w}^{(t)}) \leq e$}
    \State \Return ${w}^{(t)}$ 
\Else
    \State \Return 
    \text{"Approximately certain models do not exist"}
\EndIf
\end{algorithmic}
\end{algorithm}

\paragraph{ACM for kernel SVM} Many properties in linear regression and linear SVM, such as the linearity that is used to prove Theorem \ref{theorem:ACMedgeRepair}, do not hold for kernel SVM. Therefore, it is very challenging to efficiently compute $g(\mathbf{w}')$ and check ACM for kernel SVM. We plan to put this line of research as the future work.

\section{Experimental Evaluation}\label{sec:experimental-evaluation}
\begin{table*}[!t]
\centering
\caption{Linear SVM: Comparing Performance on Randomly Corrupted Real-World Datasets}
\vspace{-3mm}
\label{tab:LSVM-randCorrputed}
\begin{subtable}
{\textwidth}
\centering
\caption{Data Sets Where Certain Models Exist}
\label{tab:LSVM-randCorrputedCM}
\begin{tabular}{|c|c|cc|cccccc|ccccc|}
\hline
\multirow{2}{*}{Data Set} & \multirow{2}{*}{MF} & \multicolumn{2}{c|}{\textbf{Examples Cleaned}} & \multicolumn{6}{c|}{Time (Sec)} & \multicolumn{5}{c|}{Accuracy (\%)} \\
\cline{3-15}
 &  & AC & MI/KI/DI & AC & CM & DI & KI & MI & NI & AC & DI & KI & MI & CM/NI \\
\hline
 \multirow{2}{*}{Gisette} & 0.1\% & 6.07  & 14&17.48 & 2.14 & N/A & 4.21 & 1.43 & \textbf{0.90} & \textbf{97.94} & N/A & 96.60 & 97.60 & 97.40 \\
& 1\% & 60.40 & 135 & 20.32 & 2.18 & N/A & 17.88 &1.42 & \textbf{0.88} & \textbf{97.89} & N/A & 97.60 &97.60 & 97.33 \\
\hline
\multirow{3}{*}{Malware} & 0.1\% & 1.0  & \textbf{2} & 4.56 & 0.73 &N/A & 0.96 & 0.74 & \textbf{0.34} & 96.09 & N/A & \textbf{96.24} & \textbf{96.24} & \textbf{96.24} \\
& 1\% & 14.3  & 20 & 3.93 & 0.86 & N/A & 1.56 & 0.73 & \textbf{0.44} & 96.10 &  N/A & \textbf{96.24} &\textbf{96.24} & \textbf{96.24} \\
& 5\% & 44.93 & 200 & 3.17 & 0.78 & N/A & 3.99 & 0.72 & \textbf{0.36} & 96.54 &  N/A & 96.24 & 96.24 & \textbf{96.57} \\
\hline
\multirow{2}{*}{Tuandromd} & 0.1\% & 3 & 5 & 3.71 & 0.04 & 62.17 & 0.17 & 0.05 & \textbf{0.04} & 98.73 & \textbf{98.86} & 98.09 & 98.09 & 98.76 \\
& 1\% & 30.9 & 45 & 3.72 & \textbf{0.03} & 78.81 & 0.29 & 0.04 & \textbf{0.03} & 98.88 & \textbf{98.92} & 98.80 & 98.76 & 98.58 \\
\hline
\end{tabular}
\end{subtable}

\vspace{10pt}

\begin{subtable}
{\textwidth}
\centering
\caption{Data Set Where Certain Models Do Not Exist but Approximately Certain Models Exist}
\label{tab:LSVMrandCorruptednoCMbutACM}
\begin{tabular}{|c|c|cc|ccccccc|cccccc|}
\hline
\multirow{2}{*}{Data Set} & \multirow{2}{*}{MF} & \multicolumn{2}{c|}{\textbf{Examples Cleaned}} & \multicolumn{7}{c|}{Time (Sec)} & \multicolumn{6}{c|}{Accuracy (\%)} \\
\cline{3-17}
 &  & AC & MI/KI/DI & AC & CM & ACM & DI & KI & MI & NI & AC & DI & KI & MI & NI & ACM\\
\hline
\multirow{2}{*}{Gisette} 
& 5\% & 248.93 & 675 & 17.62 & 1.82 & 10.37 & N/A & 53.97 & 1.45 & \textbf{0.71} & 97.03 & N/A & \textbf{97.60} & 97.43 & 97.30 & 97.35\\
& 10\% & 393.33 & 1350 & 1.73 & 1.73 & 12.94 & N/A & 97.01 & 1.40 & \textbf{0.65} & \textbf{99.93} & N/A & 97.00 & 97.53 & 97.07 & 97.60\\
\hline
\multirow{1}{*}{Tuandromd} 
& 5\% & 91.40 & 223 & 2.41 & \textbf{0.04} & 4.31 & 74.21 & 0.44 & 0.05 & \textbf{0.04} & 98.08 & \textbf{98.76} & 98.36 & 98.36 & 98.21 & 98.18\\
\hline
\end{tabular}
\end{subtable}

\vspace{10pt}

\begin{subtable}{\textwidth}
\centering
\caption{Data Set Where Neither Certain Nor Approximately Certain Models Exist}
\label{tab:LSVMrandCorruptedneither}
\begin{tabular}{|c|c|cc|ccccccc|ccccc|}
\hline
\multirow{2}{*}{Data Set} & \multirow{2}{*}{MF} & \multicolumn{2}{c|}{\textbf{Examples Cleaned}} & \multicolumn{7}{c|}{Time (Sec)} & \multicolumn{5}{c|}{Accuracy (\%)} \\
\cline{3-16}
 &  & AC & MI/KI/DI & AC & \textbf{CM} & ACM & DI & KI & MI & NI & AC & DI & KI & MI & NI \\
\hline
\multirow{1}{*}{Malware} 
& 10\% & 66.0 & 200 & 1.97 & 0.70 & 7.16 & N/A & 6.78 & 0.73 & \textbf{0.33} & 88.05 & N/A & \textbf{96.24} & \textbf{96.24} & 83.95 \\
\hline
\multirow{1}{*}{Tuandromd} 
& 10\% & 121.33 & 446 & 1.64 & \textbf{0.04} & 4.98 & 82.45 & 0.74 & 0.05 & \textbf{0.04} & 97.36 & \textbf{98.76} & \textbf{98.76} & 98.20 & 98.54 \\
\hline
\end{tabular}
\end{subtable}

\end{table*}

\begin{table*}[!t]
\centering
\caption{p-SVM: Comparing Performance on Randomly Corrupted Real-World Datasets}
\vspace{-3mm}
\label{tab:pSVMrandCorrupted}
\begin{subtable}{\textwidth}
\centering
\caption{Data Sets Where Certain Models Exist}
\label{tab:pSVMrandCorruptedCM}
\begin{tabular}{|c|c|c|ccccc|cccc|}
\hline
\multirow{2}{*}{Data Set} & \multirow{2}{*}{MF} & \multicolumn{1}{c|}{\textbf{Examples Cleaned}} & \multicolumn{5}{c|}{Time (Sec)} & \multicolumn{4}{c|}{Accuracy (\%)} \\
\cline{3-12}
 &   & MI/KI/DI  & CM & DI & KI & MI & NI  & DI & KI & MI & CM/NI \\
\hline
\multirow{1}{*}{Gisette} & 0.1\%  & 14 & 55.69 & N/A & 56.61 & 56.01 & \textbf{55.27} &  N/A & \textbf{96.70} & \textbf{96.70} & \textbf{96.70} \\

\hline
\multirow{1}{*}{Malware} & 0.1\%   & 2& 4.98 &N/A & 4.95 & 5.01 & \textbf{4.66}  & N/A & \textbf{92.98} & \textbf{92.98} & \textbf{92.98} \\

\hline
\multirow{1}{*}{Tuandromd} & 0.1\%  & 5 &  \textbf{0.20} & 49.24 & 0.21 & 0.21 & \textbf{0.20} &  98.54 & \textbf{98.54} & \textbf{98.54} & \textbf{98.54} \\

\hline
\end{tabular}
\end{subtable}

\vspace{10pt}

\begin{subtable}{\textwidth}
\centering
\caption{Data Set Where Certain Models Do Not Exist}
\label{tab:pSVMrandCorruptednoCM}
\begin{tabular}{|c|c|c|ccccc|cccc|}
\hline
\multirow{2}{*}{Data Set} & \multirow{2}{*}{MF} & \multicolumn{1}{c|}{\textbf{Examples Cleaned}} & \multicolumn{5}{c|}{Time (Sec)} & \multicolumn{4}{c|}{Accuracy (\%)} \\
\cline{3-12}
 &  &  MI/KI/DI & \textbf{CM} & DI & KI & MI & NI & DI & KI & MI & \textbf{CM/NI} \\
\hline
\multirow{2}{*}{Gisette} & 0.1\% &  14 &  \textbf{43.79} & N/A &52.33 & 51.69 & 51.32 &  N/A & \textbf{96.80} & 96.70 & 96.70 \\
& 1\% &  135  & \textbf{42.84} & N/A & 60.04 & 51.20 & 50.04  & N/A & \textbf{96.70} & \textbf{96.70} & \textbf{96.70} \\
& 5\% &  675 &  \textbf{36.15} & N/A & 85.98 & 42.95 & 37.12 &  N/A & 96.75 & 96.70 & \textbf{96.80} \\
& 10\% &  1350 &  \textbf{36.94} & N/A & 115.80 & 45.27 & 41.63 &  N/A & 96.70 & 96.70 & \textbf{97.00} \\
\hline
\multirow{3}{*}{Malware} & 0.1\% & \textbf{2} & \textbf{4.35} & N/A & 5.64 & 5.18 & 4.35 &  N/A & \textbf{92.87} & \textbf{92.87} & 92.75 \\
& 1\% &  20 &  \textbf{4.13} & N/A & 6.76 & 4.80 & 4.53&  N/A & \textbf{92.98} & \textbf{92.98} & \textbf{92.98} \\
& 5\% & 100 & 4.84 & N/A & 7.90 & 4.81 & \textbf{4.14} &  N/A & \textbf{92.98} & \textbf{92.98} & \textbf{92.98} \\
& 10\% & 200 & \textbf{4.60} & N/A & 9.79 & 5.81 & 4.84 &  N/A & \textbf{92.98} & 92.74 & \textbf{92.98} \\
\hline
\multirow{2}{*}{Tuandromd} & 0.1\% &  5 &  0.20 & 70.22 & 0.21 & 0.21 & \textbf{0.19} & \textbf{98.54} & \textbf{98.54} & \textbf{98.54} & \textbf{98.54} \\
& 1\% &  45 & \textbf{0.18} & 44.81 & 0.30 & 0.21 & 0.19 &  \textbf{98.54} & \textbf{98.54} & \textbf{98.54} & \textbf{98.54} \\
& 5\% &  223 &  0.18 & 42.39 & 0.54 & 0.20 & \textbf{0.17} &  \textbf{98.54} & \textbf{98.54} & \textbf{98.54} & 98.31 \\
& 10\% &  446 &  0.19 & 42.45 & 0.75 & 0.21 & \textbf{0.16} &  98.54 & \textbf{98.79} & 98.54 & 98.54 \\
\hline
\end{tabular}
\end{subtable}

\end{table*}

We conduct experiments on a diverse set of real-world datasets and compare our algorithms with two natural baselines, a KNN imputation method, a deep learning-based imputation algorithm, and a benchmark method, ActiveClean. Our findings illustrate substantial savings in data cleaning costs and program running times when certain and approximately certain models exist. Moreover, our study highlights the minimal computational overhead incurred by our algorithms when verifying certain and approximately certain model conditions, even when these models do not exist.

\vspace{-3mm}
\subsection{Experimental Setup}
\subsubsection{Hardware and Platform}
We experiment on a configuration with two tasks, each utilizing two CPUs, and running on a cluster partition equipped with one 11GB GPU. The underlying hardware consists of Intel(R) Xeon(R) CPU E5-2630 v4 @ 2.20GHz machines.

\subsubsection{Real-world Datasets with randomly generated missing values}
\label{sec:information-randomly-generated-missing}

In our certain model experiments with linear SVM, polynomial SVM, and DNN(arcosine SVM) we utilize three real-world datasets. These datasets originally do not contain any missing values, but we introduce corruption by randomly injecting missing values at missing factors of $0.1\%$, $1\%$, $5\%$, and $10\%$. Where \textbf{missing factor (MF) represents the ratio of incomplete examples (examples with at least 1 missing value) to the total number of examples}. It is important to note that certain models do not exist in all versions of the corrupted datasets. Therefore, we present experimental results for both scenarios, {\it when certain models exist and when they do not}. For each dataset and each missing factor, we present the average results based on three randomly corrupted versions of the dataset in which certain models exist. This is to reduce the variability in algorithm performance resulting from the randomness of missing value injection.
\vspace{-2mm}
\paragraph{{\it Malware Dataset}}
The Malware dataset aims to distinguish between malware and benign software through the analysis of JAR files \cite{Malware:UCIDatabase}. It comprises $6825$ features and $1996$ training examples.

\paragraph{{\it Gisette Dataset}}
The Gisette dataset addresses the problem of handwritten digit recognition, with a specific focus on distinguishing between the easily confused digits $4$ and $9$ \cite{Gisette:UCIDatabase}. It consists of $13500$ training examples and $5000$ features.

\paragraph{{\it TUANDROMD Dataset}}
The TUANDROMD dataset is designed for the detection of Android malware software in contrast to benign or "goodware" applications \cite{borah2020malware}. It comprises $4464$ training examples and incorporates $241$ distinct features.

\subsubsection{Real-world Datasets originally containing missing values}\label{sec:real-world-with-missing-values}
We also conduct experiments on 8 real-world datasets originally containing missing values. Our selection includes datasets from diverse domains and {\it missing factors} (Section~\ref{sec:information-randomly-generated-missing}). \autoref{tab:real-world-missing-charateristics} presents a summary of the datasets. For preprocessing the dataset if the label is missing we drop all corresponding examples and utilize {\it sklearns} OneHotEncoder to featurize the categorical attributes. 

\paragraph{{\it Intel Sensor}} This dataset contains temperature, humidity, and light readings collected from sensors deployed in the Intel Berkeley Research lab~\cite{Intel_Sensor,krishnan2017boostclean}. The classification task is to predict whether the readings came from a particular sensor (sensor 49).

\paragraph{{\it Water Potability }}
The dataset contains information about the properties and substances (sulfate, pH) in freshwater sources, the classification task is to predict if the water is potable or not~\cite{water-potability}. 

\paragraph{{\it COVID }} This U.S. Department of Health and Human Services dataset provides data for hospital utilization dating back to January 1, 2020~\cite{healthdata-covid19}. 
The regression task is to predict the number of hospitals anticipating critical staffing shortages.

\paragraph{{\it Air Quality}} The dataset contains instances of hourly averaged responses from an array of chemical sensors embedded in an Air Quality Chemical Multisensor Device~\cite{misc_air_quality_360}. Given air-composition measurements, the regression task is to predict hourly Temperature.

\paragraph{{\it NFL}}This dataset contains play-by-play logs from US Football games. Given a play, the regression objective is to predict the score difference between the two teams~\cite{NFL2015}. We use a numeric version of this dataset since encoding string-valued attributes inflates the feature dimension by 80 times.

\paragraph{{\it Breast Cancer}}
The dataset contains information about characteristics of potential cancerous tissue and the classification task is to predict if it is benign or malignant~\cite{misc_breast_cancer_wisconsin_(original)_15}. 

\paragraph{{\it Online Education}}
This dataset contains responses to a survey on online education. The classification objective is to predict whether students prefer cellphones or laptops for online courses~\cite {OnlineEducation}. 

\paragraph{{\it Communities-Crimes}}
The data contains socio-economic and crime data from the US Census and FBI. The regression task is to predict the total number of violent crimes per 100K population~\cite{misc_communities_and_crime_183}. 



\subsubsection{Algorithms for Comparison}
\label{sec:algorithms_for_comparison}
In our experiments, we include two natural imputation baselines, a deep learning-based imputation algorithm, and a benchmark algorithm for comparison.

\noindent
\textbf{Active Clean(AC):} ActiveClean \cite{activeclean:github, krishnan2016activeclean} aims at minimizing the number of repaired examples to achieve an accurate model. We use ActiveClean for linear regression and linear SVM, but not for kernel SVM because ActiveClean's implementation relies on sklearn's SGDClassifier module, which does not support non-linear models such as kernel SVM \cite{activeclean:github}. Simply switching to sklearn's Support Vector Classification (SVC) module cannot resolve the issue since SVC does not support `partial fit', an essential function in ActiveClean. 

\noindent
\textbf{KNN-Imputer(KI):} This method predicts the values of missing items based on observed examples using a KNN classifier ~\cite{articleKNN,pmlr-v97-mattei19a}.

\noindent
\textbf{Deep-learning based Imputation (DI):}
We utilize MIWAE~\cite{pmlr-v97-mattei19a} as a sophisticated state-of-the-art imputation algorithm for comparison. This approach uses deep latent variable models to predict the values of missing data items based on the value of observed examples. Specifically, MIWAE adapts the objective of importance-weighted autoencoders~\cite{burda2016importance} and maximizes a potentially tight lower bound of the log-likelihood of the observed data.

\noindent
\textbf{Mean Imputation(MI):} MI is a widely used method for handling missing values in practice. Each missing feature value is imputed with the mean value of that feature~\cite{pmlr-v97-mattei19a}. 

\noindent
\textbf{No Imputation(NI):} NI naively drops all missing values and trains the model on the complete training set~\cite{krishnan2016activeclean}. {\it When certain model exists} the model trained with NI is equivalent to Certain Model algorithms.

\vspace{-3mm}
\subsubsection{Metrics}
We evaluate all algorithms on a held-out test set with complete examples. We use accuracy as a metric for classification tasks. Accuracy reflects the percentage of total correct class predictions therefore {\it higher accuracy is preferred}. For regression tasks, we use mean squared error (MSE). MSE measures the average squared difference between actual and predicted values. Therefore, a {\it lower MSE is preferred}. We also study the algorithms' data cleaning efforts in terms of program execution time and the number of examples cleaned.

\subsection{Results on Real-world Datasets with Random Corruption}
\label{sec:results-random-corruption}

In this section, we present results for all scenarios, {\it when certain models exist, when certain models do not exist but approximately certain models exist, and when neither exists} for 3 datasets with different degrees of random corruption (Section \ref{sec:information-randomly-generated-missing}).

\noindent
\textbf{When certain models exist :} We begin by focusing on scenarios where certain models are known to exist. Tables \ref{tab:LSVM-randCorrputedCM} and \ref{tab:pSVMrandCorruptedCM} present a performance comparison between certain model algorithms and baselines for linear SVM and p-SVM, respectively. We observe that certain models exist when the missing factor is small. This is because certain SVM models require that incomplete examples should not be support vectors in any repair (Section \ref{sec:linear-svm}). When the missing factor is small, the number of incomplete examples is also small. Therefore, the likelihood of such examples being support vectors is also small. Since certain models exist, certain model algorithms by definition (Section \ref{sec:certain-models}) clean zero missing data. 
In contrast, baseline methods spend substantial effort on data cleaning, represented by {\it examples cleaned} column in Tables \ref{tab:LSVM-randCorrputedCM} and \ref{tab:pSVMrandCorruptedCM}.
In terms of program execution time, certain model (CM) algorithms are slower than simple methods such as Mean Imputation (MI) and No Imputation (NI). However, MI and NI are both heuristics without any guarantees of the optimality of the trained model. Moreover, to find the appropriate imputation method for a specific incomplete dataset, users may want to check missing data mechanisms, which often take longer time than implementing certain model algorithms. Also, MI {\it still requires a large number of imputations}. Compared with ActiveClean (AC) and the advanced imputation methods, Deep-learning based Imputer (DI) and KNN-Imputer (KI), certain model algorithms run much faster and also guarantee an optimal model. In the tables, DI has "N/A" results for some datasets in this section when it runs for more than one hour but fails to return a result. To emphasize, {\it when certain models exist, we do not need to check for approximately certain models since they exist by definition} (Section \ref{sec:Approximate}). 



\noindent
\textbf{When certain models do not exist but approximately certain models exist}:
When certain models do not exist, we may resort to approximately certain models (ACM). Table \ref{tab:LSVMrandCorruptednoCMbutACM} shows the datasets where certain models do not exist due to the strict conditions, but approximately certain models exist and clean zero missing data. The prediction accuracy of approximately certain models is very close to the results from all baseline methods. This is because when approximately certain models exist, their approximate optimality is guaranteed. In terms of program execution time, approximately certain model algorithms run faster than DI and KI, but slower than AC, MI, and NI. However still in this scenario, checking and learning approximately certain models saves imputation costs with minimal compromise on the model's accuracy.

\noindent
\textbf{When neither certain nor approximately certain models exist:} Sometimes, neither certain nor approximately certain model may exist, as shown in Table \ref{tab:LSVMrandCorruptedneither} for linear SVM with relatively large missing factor. Since approximately certain model algorithms are not available for kernel SVMs, in Tables \ref{tab:pSVMrandCorruptednoCM} and \ref{tab:arccosrandCorrputedSVMnoCM} we present the datasets where certain models do not exist for polynomial SVM (p-SVM), and our DNN approximation with arccosine SVM, respectively. However, even if our algorithms do not find certain models for DNN, {\it certain models may still exist}. This is because the related theorem for DNN is {\it necessary but not sufficient} as described in Section \ref{sec:DNN}. We also investigate certain model existence for RBF-SVM. We observe similar patterns and results to that of arccos-SVM. Due to the limited space, we exclude the results for RBF-SVM from the paper. In scenarios when neither certain model nor approximately certain model exists, checking for them incurs some computational overhead. Nonetheless, \textit{paying for this overhead is worthwhile} for two reasons. First, substantial data-cleaning savings are realized when certain models exist (as we discuss in the previous paragraphs). Second, the time costs associated with checking certain models are minimal (confirmed by the small program running times in Tables \ref{tab:LSVMrandCorruptedneither}, \ref{tab:pSVMrandCorruptednoCM}, and \ref{tab:arccosrandCorrputedSVMnoCM}).

\begin{table*}[!t]
\centering
\caption{DNN: Comparing Performance on Randomly Corrupted Real-World Datasets: CMs are not found}
\vspace{-3mm}
\label{tab:arccosrandCorrputedSVMnoCM}

\begin{tabular}{|c|c|c|ccccc|cccc|}
\hline
\multirow{2}{*}{Data Set} & \multirow{2}{*}{MF} & \multicolumn{1}{c|}{\textbf{Examples Cleaned}} & \multicolumn{5}{c|}{Time (Sec)} & \multicolumn{4}{c|}{Accuracy (\%)} \\
\cline{3-12}
 &  &  MI/KI/DI & \textbf{CM} & DI & KI & MI & NI & DI & KI & MI & \textbf{NI} \\
\hline
\multirow{2}{*}{Gisette} & 0.1\% &  14 &  \textbf{31.27} & N/A &40.79 & 40.62 & 39.67 &  N/A & 52.30 & 52.30 & 52.10 \\
& 1\% &  135  & \textbf{35.46} & N/A & 40.69 & 40.63 & 39.57  & N/A & 52.30 & 52.30 & 52.10 \\
& 5\% &  675 &  \textbf{33.28} & N/A & 39.91 & 39.67 & 36.06 &  N/A & 52.30 & 52.30 & 52.00 \\
& 10\% &  1350 &  \textbf{29.98} & N/A & 39.94 & 39.84 & 34.10 & N/A & 52.30 & 52.30 & 51.70 \\
\hline
\multirow{3}{*}{Malware} & 0.1\% & \textbf{2} & \textbf{3.27} & N/A & 6.02 & 6.06 & 5.93 &  N/A & 48.37 & 48.37 & 48.37 \\
& 1\% &  20 &  \textbf{3.15} & N/A & 5.99 & 6.78 & 5.76&  N/A & 48.37 & 48.37 & 48.37 \\
& 5\% & 100 &  \textbf{4.03} & N/A & 5.98 & 6.00 & 5.39 &  N/A & 48.37 & 48.37 & 48.37\\
& 10\% & 200 & \textbf{4.12} & N/A & 5.98 & 5.96 & 4.92 &  N/A & 48.37 & 48.37 & 34.09 \\
\hline
\multirow{2}{*}{Tuandromd} & 0.1\% &  5 &  \textbf{0.22} & 47.56 & 0.21 & 0.28 & 0.20 & 61.81 & 61.81 & 61.81 & 61.81 \\
& 1\% &  45 & \textbf{0.25} & 69.24 & 0.21 & 0.20 & 0.20 &  61.81 & 61.81 & 61.81 & 61.81\\
& 5\% &  223 &  \textbf{0.20} & 67.94 & 0.22 & 0.21 & 0.18 &  61.81 & 61.81 & 61.81 & 61.81 \\
& 10\% &  446 &  \textbf{0.18} & 70.13 & 0.25 & 0.26 & 0.17 &  61.81 & 61.81 & 61.81 & 66.52 \\
\hline
\end{tabular}

\end{table*}

\begin{table*}[!t]
\centering
\caption{Linear Regression: Performance Comparison on Real-World Datasets with Missing Values}
\label{tab:linear-regression-performance}
\vspace{-3mm}

\begin{subtable}{\textwidth}
\centering
\caption{Data Sets Where Certain Models Exist}
\label{tab:LRcertainmodelexist}
\begin{tabular}{|c|cc|cccccc|ccccc|}
\hline
\multirow{2}{*}{Data Set}  & \multicolumn{2}{c|}{\textbf{Examples Cleaned}} & \multicolumn{6}{c|}{Time (Sec)} & \multicolumn{5}{c|}{MSE} \\
\cline{2-14}
 &  AC & MI/KI/DI & AC & CM & DI & KI & MI & NI & AC & DI & KI &MI & CM/NI\\
\hline
NFL & 12.0 & 3101 & 49.91 & 7.11 & 394.18 & 12.96 & 0.14 & \textbf{0.13} & 0.02 & \textbf{0.00} & \textbf{0.00} & \textbf{0.00} & \textbf{0.00} \\
\hline
COVID  &33.6  &32325  & 100.59 & 438.79 & 1944.10 & 587.87  &0.68  & \textbf{0.28}  &2.07 & \textbf{0.00} & \textbf{0.00} & \textbf{0.00}  & \textbf{0.00}  \\
\hline
\end{tabular}
\end{subtable}

\vspace{10pt}

\begin{subtable}{\textwidth}
\centering
\caption{Data Set Where Certain Models Do Not Exist but Approximately Certain Models Exist}
\label{tab:LRnoCMbutACM}
\begin{tabular}{|c|cc|cccccccc|cccccc|}
\hline
\multirow{2}{*}{Data Set} & \multicolumn{2}{c|}{\textbf{Examples Cleaned}} & \multicolumn{8}{c|}{Time (Sec)} & \multicolumn{6}{c|}{MSE} \\
\cline{2-17}
 & AC & MI/KI/DI & AC & CM & ACM5 & ACM6 & DI & KI & MI & NI & AC & DI & KI & MI & NI & ACM\\
\hline
Communities & 319.6 & 1494 & 2.10 & 1.45 & 4.15 & 3.74 & 4088.46 & 15.82 & 1.39 & \textbf{0.08} & 0.06 & 0.35 & 0.63 & 1.30 & 2.30 & \textbf{0.03}\\
\hline
\end{tabular}
\end{subtable}

\vspace{10pt}

\begin{subtable}{\textwidth}
\centering
\caption{Data Set Where Neither Certain Nor Approximately Certain Models Exist}
\label{tab:LRNeither}
\begin{tabular}{|c|cc|cccccccc|ccccc|}
\hline
\multirow{2}{*}{Data Set} & \multicolumn{2}{c|}{\textbf{Examples Cleaned}} & \multicolumn{8}{c|}{Time (Sec)} & \multicolumn{5}{c|}{MSE} \\
\cline{2-16}
 & AC & MI/KI/DI & AC & CM & ACM5 & ACM6 & DI & KI & MI & NI & AC & DI & KI & MI & NI\\
\hline
Air-Quality & 48.80 & 6544 & 0.87 & \textbf{0.01} &  4.62 & N/A & 111.08 & 3.67 & 0.02 & \textbf{0.01} & 28.22 & 2.11 & 3.05 & \textbf{1.07} & 3.47 \\
\hline
\end{tabular}
\end{subtable}

\end{table*}

\begin{table*}[!t]
\centering
\caption{Linear SVM: Performance Comparison on Real-World Datasets with Missing Values}
\label{tab:LinerSVMPerformance}
\vspace{-3mm}

\begin{subtable}{\textwidth}
\centering
\caption{Data Set Where Certain Models Do Not Exist but Approximately Certain Models Exist}
\label{tab:linearSVMnoCMbutACM}
\begin{tabular}{|c|cc|ccccccc|cccccc|}
\hline
\multirow{2}{*}{Data Set} & \multicolumn{2}{c|}{\textbf{Examples Cleaned}} & \multicolumn{7}{c|}{Time (Sec)} & \multicolumn{6}{c|}{Accuracy (\%)} \\
\cline{2-16}
 & AC & MI/KI/DI & AC & CM & ACM & DI & KI & MI & NI & AC & DI & KI & MI & NI & ACM\\
\hline
Intel & 30.0 & 75080 & 355.58 & 276.01 & 2428.76 & N/A &  13775.93 & 275.49 & \textbf{273.24} & 98.80 & N/A & \textbf{98.90} & 97.50 & 98.39 & 98.43\\
\hline
\end{tabular}
\end{subtable}

\vspace{10pt}

\begin{subtable}{\textwidth}
\centering
\caption{Data Set Where Neither Certain Nor Approximately Certain Models Exist}
\label{tab:linearSVMneither}
\begin{tabular}{|c|cc|ccccccc|ccccc|}
\hline
\multirow{2}{*}{Data Set} & \multicolumn{2}{c|}{\textbf{Examples Cleaned}} & \multicolumn{7}{c|}{Time (Sec)} & \multicolumn{5}{c|}{Accuracy(\%)} \\
\cline{2-15}
 & AC & MI/KI/DI & AC & CM & ACM & DI& KI & MI & NI & AC & DI& KI&MI & NI\\
\hline
Water Potability & 29.0 & 1022 & 0.34 & \textbf{0.01} &0.69 & 56.14 & 0.30 & 0.04 & \textbf{0.01} & 49.15 &\textbf{56.33}& 39.95& 39.70 & 41.89 \\
\hline
Online Education & 16.4 & 2493 & 3.88 & 0.03& 5.03 &88.87& 2.09 & 0.14 & \textbf{0.02} &  63.06 & \textbf{63.85} & 62.85& 61.37 & 36.33 \\
\hline
Breast Cancer & 9.8 & 14 & 0.06 & 0.01& 0.15 & 7.37 &0.33 &0.01 & \textbf{0.00} & 50.44 & \textbf{65.94} & \textbf{65.94} & \textbf{65.94} & 34.05 \\

\hline
\end{tabular}
\end{subtable}

\end{table*}
\vspace{-4mm}
\begin{table*}[!t]
\centering
\caption{p-SVM: Performance Comparison on Real-World Datasets with Missing Values-CMs do not exist}
\label{tab:p-SVMnoCM}
\vspace{-3mm}

\label{tab:existing-certainmodeldonotexist-regression}
\begin{tabular}{|c|c|ccccc|cccc|}
\hline
\multirow{2}{*}{Data Set} & \multicolumn{1}{c|}{\textbf{Number of Examples Cleaned}} & \multicolumn{5}{c|}{Time (Sec)} & \multicolumn{4}{c|}{Accuracy(\%)} \\
\cline{2-11}
& MI/KI/DI  & CM &  DI & KI & MI & NI  & DI & KI & MI & NI\\
\hline
Intel Sensor  & 1022 & 3321.85 & N/A & 3498.51 & 3128.54 & \textbf{3007.65} & N/A & 53.78& \textbf{58.37} & 51.98 \\
\hline
Water Potability  & 1022 & 0.05 & 41.27 & 1.03& 0.16 & \textbf{0.06} & 63.94 & 62.17& \textbf{62.96} & 61.19 \\
\hline
Online Education  & 354 & \textbf{0.18} & 57.13 & 0.98& 0.22 & 0.20& 97.13 & 95.52 &    \textbf{97.26} & 93.29 \\
\hline
Breast Cancer  & 11  & \textbf{0.01} & 42.35 & 0.25& 0.01 & 0.01 & \textbf{71.24} & 70.86 & 70.71 & 69.63 \\
\hline
\end{tabular}

\end{table*}

\begin{table*}[!t]
\centering
\caption{DNN: Performance Comparison on Real-World Datasets with Missing Values-CMs do not exist}
\label{tab:linear-regression-performance}
\vspace{-3mm}

\label{tab:arccosSVMnoCM}
\begin{tabular}{|c|c|ccccc|cccc|}
\hline
\multirow{2}{*}{Data Set} & \multicolumn{1}{c|}{\textbf{Number of Examples Cleaned}} & \multicolumn{5}{c|}{Time (Sec)} & \multicolumn{4}{c|}{Accuracy(\%)} \\
\cline{2-11}
& MI/KI/DI  & CM &  DI & KI & MI & NI  & DI & KI & MI & NI\\

\hline
Online Education  & 354 & \textbf{0.16} & 51.37 &0.79 & 0.22 & 0.20 & \textbf{42.75} & 40.96 &  41.04 & 41.06 \\
\hline
Breast Cancer  & 11  & 0.02 & 48.96 &0.02 & \textbf{0.01} & \textbf{0.01}  & 65.20 & 64.26 & \textbf{67.86} & 66.67 \\
\hline
\end{tabular}

\end{table*}

\subsection{Results on Real-world Dataset with Inherent Missingness}
\label{sec:results-original-corruption}
In our certain model experiments with linear regression, linear SVM, and SVM with kernels, we utilize 8 real-world datasets (Section \ref{sec:real-world-with-missing-values}). These datasets originally contain missing values. 

\noindent
\textbf{When certain models exist:} 
We present the result for the first scenario in Table \ref{tab:LRcertainmodelexist}. Certain models exist for NFL and COVID datasets. By checking and learning certain models with zero imputation, we save substantial energy in data cleaning compared to all 4 baselines. This imputation cost saving also comes with guarantees on the optimal model, experimentally proved by almost the same performance between certain models and the models from baseline methods. There is one exception in the COVID dataset where ActiveClean has a regression error slightly different from other baselines. This may be because partial-fit is used to proxy a complete-fit in ActiveClean's implementation \cite{activeclean:github}, which in some cases may converge early but with errors. We further investigate the data scenarios that entail certain models and verify that features irrelevant to the label are the ones with missing values. For instance, the COVID dataset receives regular data updates from 3 different sources. We observe that the newly added features are the ones with missing values.

\noindent
\textbf{When certain models do not exist but approximately certain models exist:} 
Table \ref{tab:LRnoCMbutACM} and \ref{tab:linearSVMnoCMbutACM} present the results for linear regression and linear SVM, respectively. Two datasets (Communities and Intel-Sensor) do not have certain models but have approximately certain models. Approximately certain models result in similar MSE/accuracy compared to all baseline methods, supported by the theoretical guarantee from approximately certain models (Section \ref{sec:Approximate}). In this scenario, checking and learning approximately certain models also eliminates the need for any form of data imputations. In terms of program execution time, we observe similar patterns as the results of randomly corrupted datasets. To understand the {\it influence of data characteristics on certain model existence}, we study the results for both certain models (CM and ACM). We find that certain and approximately certain models are more likely to exist in regression tasks when the number of features is large (e.g., Communities and COVID), and in classification tasks when the number of examples is large (e.g., Intel-Sensor). This is because the uncertainty from incomplete features and examples is diluted in model training when the number of features and examples is large.  

\noindent
\textbf{When neither certain nor approximately certain models exist:} 
Tables \ref{tab:LRNeither}, \ref{tab:linearSVMneither}, \ref{tab:p-SVMnoCM}, and \ref{tab:arccosSVMnoCM}
 show the cases where neither a CM nor an ACM exists (or is not found) for linear regression, linear SVM, p-SVM, and DNN, respectively. We report DNN's result (Table \ref{tab:arccosSVMnoCM}) only on two out of four classification datasets because the DNN's theorem in Section \ref{sec:DNN} only applies to datasets where each example has at most one missing value. For p-SVM and DNN, the prediction accuracy is almost the same from different baseline imputations, which often empirically suggests the existence of certain or approximately certain models. However, certain models do not exist (or are not found). To explain, when certain models do not exist, different imputations may lead to different models. Nonetheless, different models sometimes can still make identical predictions on the testing set. e. In terms of program execution time, checking certain and approximately certain models is worthwhile even if we do not find any upon checking, based on the same reasons discussed in Section \ref{sec:results-random-corruption}.
 


\section{Related Work}\label{sec2}

\noindent{\bf Stochastic and Robust Optimization.} 
Researchers have proposed stochastic optimization to find a model by optimizing the {\it expected loss function over the probability distributions of missing data items} in the training examples \cite{ganti2015sparse}. 
This approach avoids imputing missing values by redefining the loss function to include the uncertainty due to missing values in the training data.
Similarly, in robust optimization, researchers minimize the loss function of a model for the worst-case repair to an incomplete dataset, i.e., the repair that brings the highest training loss, given distributions of the missing values. 
However, the distributions of missing data items are not often available.
Thus, users may spend significant time and effort to discover or train these distributions.
Additionally, for a given type of model, users must solve various and possibly challenging optimization problems for many possible (combinations of) distributions of missing values.
In our approach, users do not need to find the probability distribution of the missing data.
Moreover, our algorithm for each type of model generalizes for all types and distributions of missing values in the training data.

\noindent{\bf Subset Selection over Incomplete Data.}
To save data cleaning costs, researchers propose to select a representative subset of training data and impute the missing values in the subset \cite{chai2023goodcore, wang2015column}. Then, a model is trained with the clean version of this subset. This approach still cleans data items. One still needs to spend time constructing a model to select a proper subset. Also, the trained model is often not the same as the model trained with the whole dataset.

{
\noindent{\bf ML Poisoning Attacks.} Researchers have proposed methods to build ML models that are robust to malicious modifications of training data to induce unwanted behavior in the model \cite{DBLP:conf/pldi/DrewsAD20}. We, however, focus on robustness against missing values in the data.}

\section{Conclusion}\label{sec9}
In this paper, we present the conditions where data repair is not
needed for training optimal and approximately optimal models over incomplete data, i.e., certain or approximately certain models exist. We also offer efficient algorithms for checking and learning certain and approximately models for linear regression, linear SVM, and kernel SVMs. Our experiments with real-world datsets demonstrate significant cost savings in data cleaning compared to five popular benchmark methods, without introducing significant overhead to the running time.

\bibliographystyle{ACM-Reference-Format}
\bibliography{mybib}


\begin{thebibliography}{34}


\ifx \showCODEN    \undefined \def \showCODEN     #1{\unskip}     \fi
\ifx \showDOI      \undefined \def \showDOI       #1{#1}\fi
\ifx \showISBNx    \undefined \def \showISBNx     #1{\unskip}     \fi
\ifx \showISBNxiii \undefined \def \showISBNxiii  #1{\unskip}     \fi
\ifx \showISSN     \undefined \def \showISSN      #1{\unskip}     \fi
\ifx \showLCCN     \undefined \def \showLCCN      #1{\unskip}     \fi
\ifx \shownote     \undefined \def \shownote      #1{#1}          \fi
\ifx \showarticletitle \undefined \def \showarticletitle #1{#1}   \fi
\ifx \showURL      \undefined \def \showURL       {\relax}        \fi
\providecommand\bibfield[2]{#2}
\providecommand\bibinfo[2]{#2}
\providecommand\natexlab[1]{#1}
\providecommand\showeprint[2][]{arXiv:#2}

\bibitem[hea(2023)]%
        {healthdata-covid19}
 \bibinfo{year}{2023}\natexlab{}.
\newblock \bibinfo{title}{{COVID-19 Reported Patient Impact and Hospital Capacity}}.
\newblock \bibinfo{howpublished}{{\url{https://catalog.data.gov/dataset/covid-19-reported-patient-impact-and-hospital-capacity-by-state-timeseries-cf58c}}}.
\newblock
\newblock
\shownote{{Accessed on 01-01-2024}}.


\bibitem[Alzubaidi et~al\mbox{.}(2021)]%
        {alzubaidi2021review}
\bibfield{author}{\bibinfo{person}{L. Alzubaidi}, \bibinfo{person}{J. Zhang}, \bibinfo{person}{A.~J. Humaidi}, \bibinfo{person}{A. Al-Dujaili}, \bibinfo{person}{Y. Duan}, \bibinfo{person}{O. Al-Shamma}, \bibinfo{person}{J. Santamar{\'i}a}, \bibinfo{person}{M.~A. Fadhel}, \bibinfo{person}{M. Al-Amidie}, {and} \bibinfo{person}{L. Farhan}.} \bibinfo{year}{2021}\natexlab{}.
\newblock \showarticletitle{Review of Deep Learning: Concepts, {CNN} Architectures, Challenges, Applications, Future Directions}.
\newblock \bibinfo{journal}{\emph{Journal of Big Data}} \bibinfo{volume}{8}, \bibinfo{number}{1} (\bibinfo{year}{2021}), \bibinfo{pages}{53}.
\newblock
\urldef\tempurl%
\url{https://doi.org/10.1186/s40537-021-00444-8}
\showDOI{\tempurl}


\bibitem[Bodik et~al\mbox{.}(2004)]%
        {Intel_Sensor}
\bibfield{author}{\bibinfo{person}{Peter Bodik}, \bibinfo{person}{Wei Hong}, \bibinfo{person}{Carlos Guestrin}, \bibinfo{person}{Sam Madden}, \bibinfo{person}{Mark Paskin}, {and} \bibinfo{person}{Romain Thibaux}.} \bibinfo{year}{2004}\natexlab{}.
\newblock \bibinfo{title}{{Intel Berkley Research Lab Data}}.
\newblock
\newblock
\urldef\tempurl%
\url{https://db.csail.mit.edu/labdata/labdata.html}
\showURL{%
\tempurl}


\bibitem[Borah et~al\mbox{.}(2020)]%
        {borah2020malware}
\bibfield{author}{\bibinfo{person}{Parthajit Borah}, \bibinfo{person}{DK Bhattacharyya}, {and} \bibinfo{person}{JK Kalita}.} \bibinfo{year}{2020}\natexlab{}.
\newblock \showarticletitle{Malware Dataset Generation and Evaluation}. In \bibinfo{booktitle}{\emph{2020 IEEE 4th Conference on Information and Communication Technology (CICT)}}. IEEE, \bibinfo{pages}{1--6}.
\newblock


\bibitem[Burda et~al\mbox{.}(2016)]%
        {burda2016importance}
\bibfield{author}{\bibinfo{person}{Yuri Burda}, \bibinfo{person}{Roger Grosse}, {and} \bibinfo{person}{Ruslan Salakhutdinov}.} \bibinfo{year}{2016}\natexlab{}.
\newblock \bibinfo{title}{Importance Weighted Autoencoders}.
\newblock
\newblock
\showeprint[arxiv]{1509.00519}~[cs.LG]


\bibitem[Chai et~al\mbox{.}(2023)]%
        {chai2023goodcore}
\bibfield{author}{\bibinfo{person}{Chengliang Chai}, \bibinfo{person}{Jiabin Liu}, \bibinfo{person}{Nan Tang}, \bibinfo{person}{Ju Fan}, \bibinfo{person}{Dongjing Miao}, \bibinfo{person}{Jiayi Wang}, \bibinfo{person}{Yuyu Luo}, {and} \bibinfo{person}{Guoliang Li}.} \bibinfo{year}{2023}\natexlab{}.
\newblock \showarticletitle{GoodCore: Data-effective and Data-efficient Machine Learning through Coreset Selection over Incomplete Data}.
\newblock \bibinfo{journal}{\emph{Proceedings of the ACM on Management of Data}} \bibinfo{volume}{1}, \bibinfo{number}{2} (\bibinfo{year}{2023}), \bibinfo{pages}{1--27}.
\newblock


\bibitem[Cho and Saul(2009)]%
        {cho2009kernel}
\bibfield{author}{\bibinfo{person}{Youngmin Cho} {and} \bibinfo{person}{Lawrence Saul}.} \bibinfo{year}{2009}\natexlab{}.
\newblock \showarticletitle{Kernel methods for deep learning}.
\newblock \bibinfo{journal}{\emph{Advances in neural information processing systems}}  \bibinfo{volume}{22} (\bibinfo{year}{2009}).
\newblock


\bibitem[Cho and Saul(2011)]%
        {cho2011analysis}
\bibfield{author}{\bibinfo{person}{Youngmin Cho} {and} \bibinfo{person}{Lawrence~K Saul}.} \bibinfo{year}{2011}\natexlab{}.
\newblock \showarticletitle{Analysis and extension of arc-cosine kernels for large margin classification}.
\newblock \bibinfo{journal}{\emph{arXiv preprint arXiv:1112.3712}} (\bibinfo{year}{2011}).
\newblock


\bibitem[Drews et~al\mbox{.}(2020)]%
        {DBLP:conf/pldi/DrewsAD20}
\bibfield{author}{\bibinfo{person}{Samuel Drews}, \bibinfo{person}{Aws Albarghouthi}, {and} \bibinfo{person}{Loris D'Antoni}.} \bibinfo{year}{2020}\natexlab{}.
\newblock \showarticletitle{Proving data-poisoning robustness in decision trees}. In \bibinfo{booktitle}{\emph{Proceedings of the 41st {ACM} {SIGPLAN} International Conference on Programming Language Design and Implementation, {PLDI} 2020, London, UK, June 15-20, 2020}}, \bibfield{editor}{\bibinfo{person}{Alastair~F. Donaldson} {and} \bibinfo{person}{Emina Torlak}} (Eds.). \bibinfo{publisher}{{ACM}}, \bibinfo{pages}{1083--1097}.
\newblock
\urldef\tempurl%
\url{https://doi.org/10.1145/3385412.3385975}
\showDOI{\tempurl}


\bibitem[Fan and Koutris(2022)]%
        {DBLP:conf/icdt/FanK22}
\bibfield{author}{\bibinfo{person}{Austen~Z. Fan} {and} \bibinfo{person}{Paraschos Koutris}.} \bibinfo{year}{2022}\natexlab{}.
\newblock \showarticletitle{Certifiable Robustness for Nearest Neighbor Classifiers}. In \bibinfo{booktitle}{\emph{25th International Conference on Database Theory, {ICDT} 2022, March 29 to April 1, 2022, Edinburgh, {UK} (Virtual Conference)}} \emph{(\bibinfo{series}{LIPIcs}, Vol.~\bibinfo{volume}{220})}, \bibfield{editor}{\bibinfo{person}{Dan Olteanu} {and} \bibinfo{person}{Nils Vortmeier}} (Eds.). \bibinfo{publisher}{Schloss Dagstuhl - Leibniz-Zentrum f{\"{u}}r Informatik}, \bibinfo{pages}{6:1--6:20}.
\newblock
\urldef\tempurl%
\url{https://doi.org/10.4230/LIPICS.ICDT.2022.6}
\showDOI{\tempurl}


\bibitem[Ganti and Willett(2015)]%
        {ganti2015sparse}
\bibfield{author}{\bibinfo{person}{Ravi Ganti} {and} \bibinfo{person}{Rebecca~M Willett}.} \bibinfo{year}{2015}\natexlab{}.
\newblock \showarticletitle{Sparse Linear regression with missing data}.
\newblock \bibinfo{journal}{\emph{arXiv preprint arXiv:1503.08348}} (\bibinfo{year}{2015}).
\newblock


\bibitem[Gentile and Warmuth(1998)]%
        {gentile1998hingeloss}
\bibfield{author}{\bibinfo{person}{Claudio Gentile} {and} \bibinfo{person}{Manfred K.~K Warmuth}.} \bibinfo{year}{1998}\natexlab{}.
\newblock \showarticletitle{Linear Hinge Loss and Average Margin}. In \bibinfo{booktitle}{\emph{Advances in Neural Information Processing Systems}}, \bibfield{editor}{\bibinfo{person}{M.~Kearns}, \bibinfo{person}{S.~Solla}, {and} \bibinfo{person}{D.~Cohn}} (Eds.), Vol.~\bibinfo{volume}{11}. \bibinfo{publisher}{MIT Press}.
\newblock


\bibitem[Horowitz(2015)]%
        {NFL2015}
\bibfield{author}{\bibinfo{person}{Max Horowitz}.} \bibinfo{year}{2015}\natexlab{}.
\newblock \bibinfo{title}{{Detailed NFL Play-by-Play Data 2015}}.
\newblock \bibinfo{howpublished}{{Kaggle}}.
\newblock
\urldef\tempurl%
\url{https://www.kaggle.com/datasets/maxhorowitz/nflplaybyplay2015}
\showURL{%
\tempurl}


\bibitem[{Isabelle Guyon, Steve Gunn, Asa Ben-Hur, Gideon Dror}(2003)]%
        {Gisette:UCIDatabase}
\bibfield{author}{\bibinfo{person}{{Isabelle Guyon, Steve Gunn, Asa Ben-Hur, Gideon Dror}}.} \bibinfo{year}{2003}\natexlab{}.
\newblock \bibinfo{title}{{Gisette}}.
\newblock
\newblock
\urldef\tempurl%
\url{https://doi.org/10.24432/C5HP5B}
\showURL{%
\tempurl}


\bibitem[Kadiwal(2021)]%
        {water-potability}
\bibfield{author}{\bibinfo{person}{Aditya Kadiwal}.} \bibinfo{year}{2021}\natexlab{}.
\newblock \bibinfo{title}{{Water Potability}}.
\newblock \bibinfo{howpublished}{{Kaggle}}.
\newblock
\urldef\tempurl%
\url{https://www.kaggle.com/datasets/adityakadiwal/water-potability}
\showURL{%
\tempurl}


\bibitem[Karla{\v{s}} et~al\mbox{.}(2020)]%
        {karlavs2020nearest}
\bibfield{author}{\bibinfo{person}{Bojan Karla{\v{s}}}, \bibinfo{person}{Peng Li}, \bibinfo{person}{Renzhi Wu}, \bibinfo{person}{Nezihe~Merve G{\"u}rel}, \bibinfo{person}{Xu Chu}, \bibinfo{person}{Wentao Wu}, {and} \bibinfo{person}{Ce Zhang}.} \bibinfo{year}{2020}\natexlab{}.
\newblock \showarticletitle{Nearest neighbor classifiers over incomplete information: From certain answers to certain predictions}.
\newblock \bibinfo{journal}{\emph{arXiv preprint arXiv:2005.05117}} (\bibinfo{year}{2020}).
\newblock


\bibitem[Krishnan et~al\mbox{.}(2017)]%
        {krishnan2017boostclean}
\bibfield{author}{\bibinfo{person}{Sanjay Krishnan}, \bibinfo{person}{Michael~J. Franklin}, \bibinfo{person}{Ken Goldberg}, {and} \bibinfo{person}{Eugene Wu}.} \bibinfo{year}{2017}\natexlab{}.
\newblock \bibinfo{title}{BoostClean: Automated Error Detection and Repair for Machine Learning}.
\newblock
\newblock
\showeprint[arxiv]{1711.01299}~[cs.DB]


\bibitem[Krishnan et~al\mbox{.}(2016)]%
        {krishnan2016activeclean}
\bibfield{author}{\bibinfo{person}{Sanjay Krishnan}, \bibinfo{person}{Jiannan Wang}, \bibinfo{person}{Eugene Wu}, \bibinfo{person}{Michael~J Franklin}, {and} \bibinfo{person}{Ken Goldberg}.} \bibinfo{year}{2016}\natexlab{}.
\newblock \showarticletitle{Activeclean: Interactive data cleaning for statistical modeling}.
\newblock \bibinfo{journal}{\emph{Proceedings of the VLDB Endowment}} \bibinfo{volume}{9}, \bibinfo{number}{12} (\bibinfo{year}{2016}), \bibinfo{pages}{948--959}.
\newblock


\bibitem[{Krishnan, Sanjay and Wang, Jiannan and Wu, Eugene and Franklin, Michael J and Goldberg, Ken}(2018)]%
        {activeclean:github}
\bibfield{author}{\bibinfo{person}{{Krishnan, Sanjay and Wang, Jiannan and Wu, Eugene and Franklin, Michael J and Goldberg, Ken}}.} \bibinfo{year}{2018}\natexlab{}.
\newblock \bibinfo{title}{{Cleaning for Data Science}}.
\newblock
\newblock
\urldef\tempurl%
\url{https://activeclean.github.io/}
\showURL{%
\tempurl}


\bibitem[Le~Morvan et~al\mbox{.}(2021)]%
        {NEURIPS2021_5fe8fdc7}
\bibfield{author}{\bibinfo{person}{Marine Le~Morvan}, \bibinfo{person}{Julie Josse}, \bibinfo{person}{Erwan Scornet}, {and} \bibinfo{person}{Gael Varoquaux}.} \bibinfo{year}{2021}\natexlab{}.
\newblock \showarticletitle{What’s a good imputation to predict with missing values?}. In \bibinfo{booktitle}{\emph{Advances in Neural Information Processing Systems}}, \bibfield{editor}{\bibinfo{person}{M.~Ranzato}, \bibinfo{person}{A.~Beygelzimer}, \bibinfo{person}{Y.~Dauphin}, \bibinfo{person}{P.S. Liang}, {and} \bibinfo{person}{J.~Wortman Vaughan}} (Eds.), Vol.~\bibinfo{volume}{34}. \bibinfo{publisher}{Curran Associates, Inc.}, \bibinfo{pages}{11530--11540}.
\newblock
\urldef\tempurl%
\url{https://proceedings.neurips.cc/paper_files/paper/2021/file/5fe8fdc79ce292c39c5f209d734b7206-Paper.pdf}
\showURL{%
\tempurl}


\bibitem[Little and Rubin(2002)]%
        {little2002statistical}
\bibfield{author}{\bibinfo{person}{R.J.A. Little} {and} \bibinfo{person}{D.B. Rubin}.} \bibinfo{year}{2002}\natexlab{}.
\newblock \bibinfo{booktitle}{\emph{Statistical analysis with missing data}}.
\newblock \bibinfo{publisher}{Wiley}.
\newblock
\showISBNx{9780471183860}
\showLCCN{2002027006}
\urldef\tempurl%
\url{http://books.google.com/books?id=aYPwAAAAMAAJ}
\showURL{%
\tempurl}


\bibitem[Liu et~al\mbox{.}(2021)]%
        {liu2021Adaptive}
\bibfield{author}{\bibinfo{person}{Tongyu Liu}, \bibinfo{person}{Ju Fan}, \bibinfo{person}{Yinqing Luo}, \bibinfo{person}{Nan Tang}, \bibinfo{person}{Guoliang Li}, {and} \bibinfo{person}{Xiaoyong Du}.} \bibinfo{year}{2021}\natexlab{}.
\newblock \showarticletitle{Adaptive Data Augmentation for Supervised Learning over Missing Data}.
\newblock \bibinfo{journal}{\emph{Proc. VLDB Endow.}} \bibinfo{volume}{14}, \bibinfo{number}{7} (\bibinfo{date}{mar} \bibinfo{year}{2021}), \bibinfo{pages}{1202–1214}.
\newblock
\showISSN{2150-8097}
\urldef\tempurl%
\url{https://doi.org/10.14778/3450980.3450989}
\showDOI{\tempurl}


\bibitem[Mattei and Frellsen(2019)]%
        {pmlr-v97-mattei19a}
\bibfield{author}{\bibinfo{person}{Pierre-Alexandre Mattei} {and} \bibinfo{person}{Jes Frellsen}.} \bibinfo{year}{2019}\natexlab{}.
\newblock \showarticletitle{{MIWAE}: Deep Generative Modelling and Imputation of Incomplete Data Sets}. In \bibinfo{booktitle}{\emph{Proceedings of the 36th International Conference on Machine Learning}} \emph{(\bibinfo{series}{Proceedings of Machine Learning Research}, Vol.~\bibinfo{volume}{97})}, \bibfield{editor}{\bibinfo{person}{Kamalika Chaudhuri} {and} \bibinfo{person}{Ruslan Salakhutdinov}} (Eds.). \bibinfo{publisher}{PMLR}, \bibinfo{pages}{4413--4423}.
\newblock
\urldef\tempurl%
\url{https://proceedings.mlr.press/v97/mattei19a.html}
\showURL{%
\tempurl}


\bibitem[Neutatz et~al\mbox{.}(2021)]%
        {neutatz2021cleaning}
\bibfield{author}{\bibinfo{person}{Felix Neutatz}, \bibinfo{person}{Binger Chen}, \bibinfo{person}{Ziawasch Abedjan}, {and} \bibinfo{person}{Eugene Wu}.} \bibinfo{year}{2021}\natexlab{}.
\newblock \showarticletitle{From Cleaning before ML to Cleaning for ML.}
\newblock \bibinfo{journal}{\emph{IEEE Data Eng. Bull.}} \bibinfo{volume}{44}, \bibinfo{number}{1} (\bibinfo{year}{2021}), \bibinfo{pages}{24--41}.
\newblock


\bibitem[Picado et~al\mbox{.}(2020)]%
        {picado2020learning}
\bibfield{author}{\bibinfo{person}{Jose Picado}, \bibinfo{person}{John Davis}, \bibinfo{person}{Arash Termehchy}, {and} \bibinfo{person}{Ga~Young Lee}.} \bibinfo{year}{2020}\natexlab{}.
\newblock \showarticletitle{Learning over dirty data without cleaning}. In \bibinfo{booktitle}{\emph{Proceedings of the 2020 ACM SIGMOD International Conference on Management of Data}}. \bibinfo{pages}{1301--1316}.
\newblock


\bibitem[Redmond(2009)]%
        {misc_communities_and_crime_183}
\bibfield{author}{\bibinfo{person}{Michael Redmond}.} \bibinfo{year}{2009}\natexlab{}.
\newblock \bibinfo{title}{{Communities and Crime}}.
\newblock \bibinfo{howpublished}{{UCI Machine Learning Repository}}.
\newblock


\bibitem[{Ricardo P Pinheiro, Sidney M. L. Lima, Sérgio M. M. Fernandes, E. D. Q. Albuquerque, S. Medeiros, Danilo Souza, T. Monteiro, Petrônio Lopes, Rafael Lima, Jemerson Oliveira, Sthéfano Silva}(2019)]%
        {Malware:UCIDatabase}
\bibfield{author}{\bibinfo{person}{{Ricardo P Pinheiro, Sidney M. L. Lima, Sérgio M. M. Fernandes, E. D. Q. Albuquerque, S. Medeiros, Danilo Souza, T. Monteiro, Petrônio Lopes, Rafael Lima, Jemerson Oliveira, Sthéfano Silva}}.} \bibinfo{year}{2019}\natexlab{}.
\newblock \bibinfo{title}{{REJAFADA}}.
\newblock
\newblock
\urldef\tempurl%
\url{https://doi.org/10.24432/C5HG8D}
\showURL{%
\tempurl}


\bibitem[RUBIN(1976)]%
        {10.1093/biomet/63.3.581}
\bibfield{author}{\bibinfo{person}{DONALD~B. RUBIN}.} \bibinfo{year}{1976}\natexlab{}.
\newblock \showarticletitle{{Inference and missing data}}.
\newblock \bibinfo{journal}{\emph{Biometrika}} \bibinfo{volume}{63}, \bibinfo{number}{3} (\bibinfo{date}{12} \bibinfo{year}{1976}), \bibinfo{pages}{581--592}.
\newblock
\showISSN{0006-3444}
\urldef\tempurl%
\url{https://doi.org/10.1093/biomet/63.3.581}
\showDOI{\tempurl}
\showeprint{https://academic.oup.com/biomet/article-pdf/63/3/581/756166/63-3-581.pdf}


\bibitem[Troyanskaya et~al\mbox{.}(2001)]%
        {articleKNN}
\bibfield{author}{\bibinfo{person}{Olga Troyanskaya}, \bibinfo{person}{Mike Cantor}, \bibinfo{person}{Gavin Sherlock}, \bibinfo{person}{Trevor Hastie}, \bibinfo{person}{Rob Tibshirani}, \bibinfo{person}{David Botstein}, {and} \bibinfo{person}{Russ Altman}.} \bibinfo{year}{2001}\natexlab{}.
\newblock \showarticletitle{Missing Value Estimation Methods for DNA Microarrays}.
\newblock \bibinfo{journal}{\emph{Bioinformatics}}  \bibinfo{volume}{17} (\bibinfo{date}{07} \bibinfo{year}{2001}), \bibinfo{pages}{520--525}.
\newblock
\urldef\tempurl%
\url{https://doi.org/10.1093/bioinformatics/17.6.520}
\showDOI{\tempurl}


\bibitem[Van~Buuren(2018)]%
        {van2018flexible}
\bibfield{author}{\bibinfo{person}{Stef Van~Buuren}.} \bibinfo{year}{2018}\natexlab{}.
\newblock \bibinfo{booktitle}{\emph{Flexible imputation of missing data}}.
\newblock \bibinfo{publisher}{CRC press}.
\newblock


\bibitem[Vanschoren et~al\mbox{.}(2013)]%
        {OnlineEducation}
\bibfield{author}{\bibinfo{person}{Joaquin Vanschoren}, \bibinfo{person}{Jan~N. van Rijn}, \bibinfo{person}{Bernd Bischl}, {and} \bibinfo{person}{Luis Torgo}.} \bibinfo{year}{2013}\natexlab{}.
\newblock \showarticletitle{OpenML: Networked Science in Machine Learning}.
\newblock \bibinfo{journal}{\emph{SIGKDD Explorations}} \bibinfo{volume}{15}, \bibinfo{number}{2} (\bibinfo{year}{2013}), \bibinfo{pages}{49--60}.
\newblock
\urldef\tempurl%
\url{https://doi.org/10.1145/2641190.2641198}
\showDOI{\tempurl}


\bibitem[Vito(2016)]%
        {misc_air_quality_360}
\bibfield{author}{\bibinfo{person}{Saverio Vito}.} \bibinfo{year}{2016}\natexlab{}.
\newblock \bibinfo{title}{Air Quality}.
\newblock \bibinfo{howpublished}{UCI Machine Learning Repository}.
\newblock


\bibitem[Wang and Singh(2015)]%
        {wang2015column}
\bibfield{author}{\bibinfo{person}{Yining Wang} {and} \bibinfo{person}{Aarti Singh}.} \bibinfo{year}{2015}\natexlab{}.
\newblock \showarticletitle{Column subset selection with missing data via active sampling}. In \bibinfo{booktitle}{\emph{Artificial Intelligence and Statistics}}. PMLR, \bibinfo{pages}{1033--1041}.
\newblock


\bibitem[Wolberg(1992)]%
        {misc_breast_cancer_wisconsin_(original)_15}
\bibfield{author}{\bibinfo{person}{William Wolberg}.} \bibinfo{year}{1992}\natexlab{}.
\newblock \bibinfo{title}{{Breast Cancer Wisconsin (Original)}}.
\newblock \bibinfo{howpublished}{{UCI Machine Learning Repository}}.
\newblock


\end{thebibliography}


\end{document}